%% file: main.tex
\documentclass[lettersize,journal]{IEEEtran}
\input{symbols_commands}

\usepackage{amsmath,amsfonts}
\usepackage{amsthm}
\usepackage{cite}
\usepackage{url}
\usepackage{xcolor}  % colors
\usepackage{hyperref}
\usepackage{enumitem}
\usepackage[font=small]{caption}
\usepackage{array}
\usepackage{subcaption}
\usepackage{textcomp}
\usepackage{stfloats}
\usepackage{url}
\usepackage{verbatim}
\usepackage{graphicx}
\usepackage{booktabs}
\usepackage{titletoc}
\usepackage[linesnumbered,ruled,vlined]{algorithm2e}
\allowdisplaybreaks
\definecolor{lightred}{rgb}{1, 0.7, 0.7}
\definecolor{lightgreen}{rgb}{0.7, 1, 0.7}

\newcommand{\com}[1]{\tiny$\pm$#1}

\begin{document}

\title{Communication-Efficient and Differentially Private Vertical Federated Learning with Zeroth-Order Optimization}

\author{Jianing Zhang,~\IEEEmembership{Student Member,~IEEE,} Evan Chen,~\IEEEmembership{Student Member,~IEEE,} Dong-Jun Han,~\IEEEmembership{Member,~IEEE,} Chaoyue Liu,~\IEEEmembership{Member,~IEEE,} Christopher G. Brinton,~\IEEEmembership{Senior Member,~IEEE}
        % <-this % stops a space
\thanks{J. Zhang, E. Chen, C. Liu, and C. Brinton are with the Elmore Family School of Electrical and Computer Engineering, Purdue University, West Lafayette, IN 47907, USA. Email: \{zhan4670, chen4388, cyliu, cgb\}@purdue.edu.}%
\thanks{D.-J. Han is with the Department of Computer Science and Engineering, Yonsei University, Seoul 03722, Republic of Korea. Email: djh@yonsei.ac.kr.}
}

        % <-this % stops a space


% THEOREMS
%%%%%%%%%%%%%%%%%%%%%%%%%%%%%%%%
\theoremstyle{plain}
\newtheorem{theorem}{Theorem}[section]
\newtheorem{result}[theorem]{Result}

\newtheorem{proposition}[theorem]{Proposition}
\newtheorem{lemma}[theorem]{Lemma}
\newtheorem{corollary}[theorem]{Corollary}
\theoremstyle{definition}
\newtheorem{assumption}[theorem]{Assumption}
\newtheorem{definition}[theorem]{Definition}
\theoremstyle{remark}
\newtheorem{remark}[theorem]{Remark}
\newtheorem{fact}{Fact}
\maketitle

\begin{abstract}
% Vertical Federated Learning (VFL) enables collaborative training with feature-partitioned data, yet remains vulnerable to private label leakage through gradient transmissions. In this work, we propose {\tt DPZV}, a gradient-free VFL framework that achieves tunable differential privacy (DP) with formal performance guarantees. By leveraging zeroth-order (ZO) optimization, {\tt DPZV} eliminates explicit gradient exposure. It further enhances security by providing provable differential privacy guarantees. Standard DP techniques like DP-SGD are difficult to apply in zeroth-order VFL due to VFL’s distributed nature and the high variance incurred by vector-valued noise. {\tt DPZV} overcomes these limitations by injecting low-variance scalar noise at the server, enabling controllable privacy with reduced memory overhead. We conduct a comprehensive theoretical analysis showing that {\tt DPZV} attains convergence rate comparable to first order (FO) optimization methods while satisfying formal $(\epsilon, \delta)$-DP guarantees. Experiments on image and language benchmarks demonstrate that \texttt{DPZV} outperforms several baselines in terms of achieved accuracy under a wide range of privacy constraints ($\epsilon \leq 10$), thereby elevating the privacy-utility tradeoff in VFL.
Vertical Federated Learning (VFL) enables collaborative model training across feature-partitioned devices, yet its reliance on device-server information exchange introduces significant communication overhead and privacy risks. Downlink communication from the server to devices in VFL exposes gradient-related signals of the global loss that can be leveraged in inference attacks. Existing privacy-preserving VFL approaches that inject differential privacy (DP) noise on the downlink have the natural repercussion of degraded gradient quality, slowed convergence, and excessive communication rounds. In this work, we propose {\tt DPZV}, a communication-efficient and differentially private ZO-VFL framework with tunable privacy guarantees. Based on zeroth-order (ZO) optimization, {\tt DPZV} injects calibrated scalar-valued DP noise on the downlink, significantly redu cing variance amplification while providing equivalent protection against targeted inference attacks. Through rigorous theoretical analysis, we establish convergence guarantees comparable to first-order DP-SGD, despite relying solely on ZO estimators, and prove that {\tt DPZV} satisfies $(\epsilon,\delta)$-DP. Extensive experiments demonstrate that {\tt DPZV} consistently achieves a superior privacy–utility tradeoff and requires fewer communication rounds than existing DP-VFL baselines under strict privacy constraints ($\epsilon \leq 10$). 
\end{abstract}


\section{Introduction}
% \evan{Add more communication related content. Connection of VFL and Zeroth-order to communication}\evan{Focus the story on why we want to focus on downlink protection (i.e. the uplink is usually not eavedropped)}

% \evan{Don't focus on why we have unbalanced updates. Focus on: 1. Why protecting downlink is important. 2. Why saving total communication rounds is more important (try to make the reviewers ignore the fact tha we have unbalanced per round update.)}

Vertical Federated Learning (VFL) has emerged as a compelling paradigm for distributed training over edge and wireless systems where devices hold complementary features for the same data samples~\cite{hardy2017private,chen2020vafl,castiglia2023flexible}. Such vertical feature partitioning naturally arises in many applications. For example, in wireless sensor networks, spatially-distributed nodes may each collect local environmental measurements treated as features in a learning task~\cite{wang2018edge,shi2016edge,park2019wireless,guo2022nonparametric}. In other distributed setups, devices may only be capable of storing a portion of the features due to strict memory budgets or security concerns~\cite{yu2020low,liu2024vertical}. In these settings, VFL enables flexible model deployment by allowing learning to be distributed across devices with only a server/fusion center storing and processing the full model.

In most VFL systems, different devices holding partial features exchange intermediate gradients or embeddings with the server for aggregation~\cite{liu2020backdoor, chen2020vafl,weng2020practical}. This is in contrast to conventional (horizontal) FL \cite{mcmahan2017communication, wang2018edge, 10.1109/MWC.011.2000501,9276464}, where each device holds a full feature vector for a subset of the samples, and updates are exchanged on the full model. The result is a different communication pattern in VFL, where information exchange typically occurs with respect to smaller submodel partitions rather than through full model updates. Still, communication efficiency is a pressing issue in VFL, influenced heavily by the feature dimensionality and training procedure~\cite{castiglia2023less,liu2024vertical}.

Beyond efficiency concerns, VFL’s reliance on transmitting intermediate results introduces privacy concerns. Recent studies have shown that adversaries can exploit communicated information to infer sensitive attributes (feature inference attacks)~\cite{jin2021cafe,ye2024feature} or even recover labels (label inference attacks)~\cite{fu2022label,zou2022defending} and leverage them in e.g., backdoor attacks~\cite{lee2025cooperative}. Under the commonly adopted honest-but-curious threat model, these risks are particularly pronounced on the downlink, where gradient-related signals directly encode label-dependent information. Therefore, protecting downlink communication is essential for safeguarding privacy in VFL.

Unfortunately, existing mitigation measures for these privacy concerns often harm model performance and further exacerbate the communication burden of VFL. Similar to conventional FL, VFL methods often protect gradients by injecting calibrated noise vectors using differential privacy (DP) mechanisms~\cite{chen2020vafl,wang2024unified,xie2024improving,xu2019hybridalpha}. However, such noise inevitably degrades gradient quality, leading to slower convergence and thus requiring more communication rounds to reach a target accuracy. Consequently, privacy protection and communication efficiency are tightly coupled in DP-enhanced VFL (DP-VFL): stronger privacy requirements for a fixed target model quality typically translate into heavier communication costs.

\subsection{Zeroth-Order VFL}
Zero-order (ZO) optimizers~\cite{nesterov2017random, fang2022communication} offer a compelling foundation for building private and efficient solutions to VFL~\cite{zhang2021desirable,wang2024unified}. Migrating from first-order gradient descent-based methods to zero-order VFL (ZO-VFL) is motivated by two key observations. First, ZO methods eliminate explicit gradient transmission in the backward phase, substantially reducing information leakage and providing a stronger baseline defense against label inference attacks~\cite{zhang2021desirable}. Second, ZO enables optimization using only function evaluations. As a result, the backward communication in ZO-VFL consists of scalar-valued feedback rather than high-dimensional gradients. This significantly limits the dimensionality of information transmitted and potentially revealed per downlink interaction.

However, ZO-VFL alone does not fully resolve privacy concerns. Despite the absence of explicit gradients, malicious devices can still approximate gradients from perturbation-based scalar feedback, leaving ZO-VFL vulnerable to inference attacks. Moreover, while recent work \cite{wang2024unified} has quantified an inherent differential privacy provided by ZO methods, these levels are typically insufficient when strong DP guarantees are needed \cite{gupta2024inherent}. Importantly, these DP levels are also not adjustable, limiting their practicality in real-world deployments that require controllable privacy guarantees.

These limitations lead to the central question of this work:
\begin{itemize}
\item[ ] \textbf{\textit{How can VFL obtain communication efficiency advantages as in ZO-VFL while preserving controllable privacy levels as in DP-VFL?}}
\end{itemize}

%Among privacy mechanisms, differential privacy (DP) ~\cite{dwork2006calibrating} stands out as it introduces a rigorous privacy guarantee against a broad class of inference attacks ~\citep{Zhao2021LDP,Shen2022imp,Liu2023diff}, including feature inference. Moreover, DP allows for adjustable privacy levels under varying privacy budgets, enabling flexible trade-offs between model performance and privacy~\citep{dwork2014algorithmic}. This makes DP a natural candidate for strengthening ZO-based VFL.

% To answer this question, we observe that achieving communication efficiency in VFL requires controlling both the per-round communication overhead (PRCO) and the number of communication rounds required for convergence. However, accomplishing this goal faces two fundamental challenges. First, existing strategies for achieving differential privacy in VFL typically inject noise into the forward embeddings~\cite{chen2020vafl,xie2024improving}, which are typically high-dimensional. Thus, the high-dimensional injected noise leads to substantial variance amplification and significantly slowing convergence, especially under tight privacy budgets. Second, ZO methods inherently rely on noisy gradient estimators and are generally believed to converge more slowly than first-order(FO) methods~\cite{ghadimi2013stochastic}. When combined with additional DP noise, this issue can be further exacerbated, potentially resulting in severe performance degradation.
\noindent To answer this question, we observe that achieving communication efficiency in DP-VFL is fundamentally tied to obtaining a favorable privacy–accuracy tradeoff, since improved convergence behavior directly translates into fewer communication rounds needed to reach a target accuracy. However, accomplishing this goal faces two key challenges. First, \textit{existing approaches for enforcing differential privacy in VFL commonly inject noise into the forward embeddings}~\cite{chen2020vafl,xie2024improving}, which are typically high-dimensional. As a result, the injected noise induces substantial variance amplification and significantly slows convergence, particularly under tight privacy budgets. Second, \textit{ZO optimization methods rely on inherently noisy gradient estimators and are generally known to converge more slowly than first-order methods}~\cite{ghadimi2013stochastic}. When combined with additional DP noise, this effect can be further amplified, potentially leading to severe performance degradation and an increased number of communication rounds.

\subsection{Summary of Contributions}
In this work, we propose {\tt DPZV}, a VFL methodology which achieves communication efficiency and controllable differential privacy concurrently. Instead of perturbing the forward embeddings, our method injects scalar-valued DP noise into the backward information. Intuitively, injecting scalar noise preserves the directional structure of the ZO gradient estimator and introduces far less distortion than vector-valued perturbations, enabling more accurate updates and faster convergence under the same privacy budget. We show in Sec.~\ref{sec:privacy} that this design achieves the same level of protection against targeted label and feature inference attacks as forward perturbations, while significantly reducing noise dimensionality. 

Through both theoretical analysis and extensive experiments, we demonstrate that our approach provides faster convergence than existing methods under a fixed privacy budget, incurring the lowest total communication cost to reach a target accuracy. More specifically, our contributions are as follows:
\begin{itemize}%[topsep=0pt,itemsep=4pt,parsep=0pt, partopsep=0pt, leftmargin=*]
[leftmargin=*]
    \item {\tt DPZV} is the first VFL methodology that enables tunable differential privacy while maintaining communication efficiency. Unlike existing DP-VFL methods that inject high-dimensional noise into forward embeddings, {\tt DPZV} perturbs only the scalar-valued backward information. This design preserves privacy guarantees while substantially reducing variance in the gradient estimation, leading to faster convergence and fewer required communication rounds. By eliminating gradient transmission, {\tt DPZV} is particularly well suited for privacy-critical and communication-constrained deployments (Sec. \ref{sec:method}).
    
    \item We provide a rigorous theoretical analysis establishing both convergence and privacy guarantees for \texttt{DPZV}.
    Specifically, we show that \texttt{DPZV} achieves a convergence rate to a stationary point on the same order as first-order DP methods, despite using less precise ZO estimators. This result demonstrates that the conventional gap in convergence speed between ZO and first-order methods diminishes in the DP-VFL setting. Furthermore, we prove that our method satisfies $(\epsilon, \delta)$-DP, confirming its ability to provide an adjustable privacy control mechanism while retaining communication efficiency (Sec. \ref{sec:convergence}, \ref{sec:privacy}).

    \item Through experiments on four benchmark datasets, we verify that {\tt DPZV} obtains robust convergence performance under strict privacy budgets. In particular, while the baseline algorithms experience a steep performance degradation and/or increase in communication rounds as the privacy level increases, we find that {\tt DPZV} consistently maintains faster convergence to a target accuracy across all tasks, underscoring an 
    elevated privacy-utility-cost tradeoff. In addition to a lower total communication cost, we find that {\tt DPZV} reduces the required memory overhead for training large models (Sec. \ref{sec:exp}).
\end{itemize}

\begin{figure*}[t]
\centering
\includegraphics[width=0.98\linewidth]{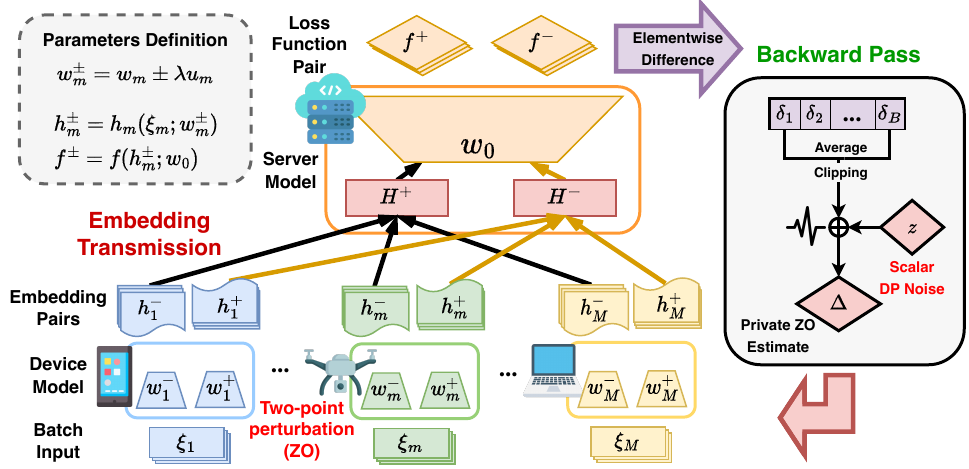}
    \caption{
Overview of the training procedure in {\tt DPZV}. Each device perturbs its local model parameters in two random directions to generate a pair of embeddings, which are then transmitted to the server. The server computes the corresponding function evaluations and applies an elementwise difference to approximate the zeroth-order (ZO) gradient. To ensure differential privacy, scalar-valued Gaussian noise is injected into the aggregated ZO estimate. Unlike traditional vector-valued noise in standard DP algorithms, scalar noise is significantly smaller in norm, thereby preserving model utility even under stringent privacy budgets.
    }
    \label{fig:visualize}
    \vspace{-10pt}
\end{figure*}

\section{Related Work}
%We now discuss related work along two dimensions: (i) vertical federated learning and its privacy-preserving variants, and (ii) zeroth-order optimization.

\subsection{Vertical Federated Learning}
VFL enables collaborative training across entities (e.g., organizations or, in our case, wireless devices) with vertically partitioned features. Early VFL frameworks focused on simple device-side models such as logistic regression and linear models~\cite{hardy2017private}. These methods prioritized simplicity, but lacked expressiveness for complex tasks. To address this limitation, larger device-side models like deep neural networks (DNNs) were adopted~\cite{chen2020vafl, castiglia2023flexible, xie2024improving}.

A key challenge in VFL is the communication overhead incurred during training. One popular mechanism for reducing communication overhead is by allowing for multiple local updates in-between aggregations. In this respect, \cite{liu2022fedbcd} introduced {\tt FedBCD},  which allows devices to perform multiple gradient iterations in VFL before synchronization. Similarly, {\tt Flex-VFL}~\cite{castiglia2023flexible} proposed a flexible strategy offering varying local update counts per party, constrained by a communication timeout. {\tt VIMADMM}~\cite{xie2024improving} adopted an ADMM-based approach to enable multiple local updates in VFL. On the other hand, Asynchronous VFL methods (e.g., {\tt FDML}~\cite{hu2019fdml}, {\tt VAFL}~\cite{chen2020vafl}) decouple coordination, allowing devices to update models independently, thus improving scalability. However, in first-order methods, the backward pass through neural networks typically imposes communication overhead, while our ZO-based approach significantly reduces the cost associated with backward propagation.

Privacy guarantees are another critical challenge for VFL adoption. Some VFL architectures use crypto-based privacy-preserving techniques such as Homomorphic Encryption (HE)~\cite{cheng2021secureboost}, but lack formal assurances. In contrast, DP provides rigorous mathematical protection. Key DP-based methods include {\tt VAFL}~\cite{chen2020vafl}, which injects Gaussian noise into device embeddings during forward propagation to achieve Gaussian DP, and {\tt VIMADMM}~\cite{xie2024improving}, which perturbs linear model parameters with bounded sensitivity, ensuring DP guarantees for convex settings. Our work overcomes challenges in developing such methods for the ZO setting.

\subsection{Zeroth-Order Optimization}  
% \textbf{Zeroth-Order Optimization.}
Recent research has explored ZO optimization within VFL to accommodate resource-constrained devices with non‐differentiable models and to reduce gradient leakage. Early work like {\tt ZOO-VFL} \cite{zhang2021desirable} adopted a naive ZO approach throughout VFL training but provided no DP guarantees. {\tt VFL-CZOFO}~\cite{wang2024unified} introduced a cascaded hybrid optimization method, combining zeroth-order and first-order updates, which leveraged intrinsic noise from ZO for limited privacy. However, its DP level was not adjustable, resulting in insufficient protection.

More recently, {\tt MeZO}~\cite{malladi2023fine} proposed a memory-efficient ZO algorithm. Building upon these ideas, {\tt DPZero}~\cite{zhang2024dpzero} and {\tt DPZO}~\cite{tang2024private} introduced private ZO variants offering baseline privacy features. However, these methods are designed for centralized settings and cannot be directly extended to the VFL paradigm. Extending beyond previous efforts to combine ZO optimization with VFL, we further integrate a controllable differential privacy mechanism, achieving an elevated trade-off between privacy and model performance. 

\section{Preliminaries}
\label{sec:background_and_motivation}
\subsection{System Model for VFL}
We consider a VFL framework with one server and $M$ devices. In VFL, data is vertically partitioned between devices, with each device holding different features of the data. Suppose we have a dataset $\gD$ with $D$ samples: $\gD=\{\xi_i|i=1,2,\dots,D\}$. Each data sample $\xi_i$ can be partitioned into $M$ portions distributed throughout all devices, where the data sample on machine $m$ with ID $i$ is denoted as $\xi_{m,i}$, hence $\xi_i = [\xi_{1,i}, \xi_{2,i}, \ldots, \xi_{M,i}]^\top$. The server is numbered as machine $0$ and holds the label data $\mY = \{y_i|i=1,2,\dots,D\}$.

The devices and the server aim to jointly train a machine learning model parameterized by $\vw$. The global model comprises local models on each party parameterized by $w_m$, with $m= 0,1,\dots, M$ being the ID of the machine. To protect privacy, devices do not communicate with each other regarding data or model parameters. Instead, all devices communicate directly with the server regarding their local model's outputs, which we term as local embeddings. If we denote the local embedding of device $m$ as $h_{m,i}:=h(w_m;\xi_{m,i})$, the objective of the VFL framework can be seen as minimizing the following function:
\begin{align}\label{eq:defition}
     F(\vw;\gD,\mY):=\frac{1}{D}\sum_{i\in[D]}\gL(w_0, h_{1,i}, h_{2,i},\dots, h_{M,i};y_i)
\end{align}
 where $\gL$ is the loss function for a datum $\xi_i$ and its corresponding label $y_i$. For the simplicity of notation, we define the loss function w.r.t a specific model parameter and datum as \begin{align}\label{eq:f}
      f(\vw;\xi_i):=\gL(w_0, h_{1,i}, h_{2,i},\dots, h_{M,i};y_i)
 \end{align}
\textbf{Memory-Efficient Zeroth-Order Optimizer.}
% \label{ssec:ZOO} 
% Zeroth-Order Optimization has been widely discussed under both convex and non-convex conditions~\citep{ghadimi2013stochastic,nesterov2017random,fang2022communication}. This method finds the optima of an objective function without relying on first-order information. Such properties are especially useful when first-order gradient is computationally expensive or intractable. 
We introduce the two-point gradient estimator~\cite{shamir2017optimal}, which will serve as our zeroth-order gradient estimator throughout this paper. 
Let $u$ be uniformly sampled from the Euclidean sphere $\sqrt{d}\sS^{d-1}$ and $\lambda>0$ be the smoothing factor. For a single datum $\xi$ sampled from the whole dataset $\gD$, the two-point gradient estimator is defined as
    \begin{align}\label{eq:zero}
      g(x;\xi) = \frac{f(x+\lambda u;\xi)-f(x-\lambda u;\xi)}{2\lambda}u
    \end{align}

To further reduce the memory overhead on device, we adopt the {\tt MeZO} methodology~\cite{malladi2023fine} for our ZO Optimization, which requires only the same memory as the model itself for training. We investigate the memory reduction of this estimator in Section \ref{sec:exp}.

\subsection{Differential Privacy and DP-SGD}
DP provides a principled framework for protecting individual data in statistical analysis and machine learning. Formally, we adopt the standard $(\epsilon, \delta)$-DP definition:

\begin{definition}[$(\epsilon, \delta)$-Differential Privacy] A randomized algorithm $\gM: \gX^n \rightarrow \Theta$ is said to satisfy $(\epsilon, \delta)$-differential privacy if for any neighboring datasets $X, X' \in \gX^n$ that differ by only one individual's data, and for any subset of outputs $E \subseteq \Theta$, we have \begin{align} \Prob[\gM(X) \in E] \le e^\epsilon \Prob[\gM(X') \in E] + \delta. \label{eq:DP}\end{align} \end{definition}
One of the most widely adopted private training algorithms is Differentially Private Stochastic Gradient Descent (DP-SGD), which operates by clipping an intermediate result and injecting Gaussian noise during each model update. Specifically, 
given a minibatch $B_t$ at iteration $t$, DP-SGD performs the following steps:
\begin{enumerate}
\item Clip individual gradients: \\ $\bar{g}_i = \nabla \ell(\theta_t, x_i), \quad \tilde{g}_i = \bar{g}_i / \max\left(1, \frac{\|\bar{g}_i\|_2}{C}\right)$

\item Add Gaussian noise: \\ $ \tilde{g}_t = \frac{1}{|B_t|} \left( \sum_{i \in B_t} \tilde{g}_i + \mathcal{N}(0, \sigma^2 C^2 \mathbf{I}) \right)$

\item Update model: $\theta_{t+1} = \theta_t - \eta_t \tilde{g}_t$
\end{enumerate}

\subsection{Challenges and Motivation}
Existing approaches~\cite{chen2020vafl, xie2024improving} integrate DP with first order VFL by injecting DP noise in the forward embedding. In centralized or horizontally federated learning, DP-SGD achieves privacy by adding noise to the backward gradients. However, extending these strategies to ZO optimization introduces new challenges. First, \textit{due to the distributed nature of VFL, DP guarantees must be enforced during the downlink process between the cloud and devices, so as to prevent malicious devices from mounting label inference attacks.} In both cases, the noise must match the dimensionality of the protected vector (embedding or gradient), often resulting in high-dimensional noise. This situation is further exacerbated by the second factor: in ZO methods, gradients are approximated through multiple forward evaluations, making them inherently noisy. Adding a \textit{high-dimensional} noise vector as in standard DP-SGD further amplifies this inaccuracy. 

A key advantage of adopting ZO optimization, however, lies in its communication efficiency in the backward pass: only a scalar value, rather than a full gradient vector, needs to be exchanged. This observation enables \textit{a more efficient privacy mechanism—injecting scalar noise rather than high-dimensional noise}. Building on this insight, we propose Differentially Private Zeroth-Order Vertical Federated Learning ({\tt DPZV}), a framework that achieves tunable differential privacy through calibrated scalar noise injection. This design allows for a favorable privacy-utility trade-off in ZO-based VFL, particularly under tight privacy budgets. Unlike first-order VFL methods, which typically add noise to forward embeddings, {\tt DPZV} introduces noise in the backward pass, targeting the most exploited leakage pathway in VFL: label inference from gradients. A broader discussion of privacy implications is provided in Section \ref{sec:privacy}.

\section{Methodology}
\label{sec:method}
In this section, we describe the overall training procedure of {\tt DPZV}.  Based on \eqref{eq:defition}, the objective is to collaboratively minimize the global objective function across all devices. The devices hold disjoint features of the same data records, and the server holds the label data. The training procedure proceeds in two iterative steps.

\subsection{Device Update and Forward Communication}
The sampled devices compute local information and transmit it to the server. We define the device-server communication as \textit{forward communication}.

\textbf{Training Procedure.} Each device $m$ maintaining its own local communication round $t_m$, while the server tracks a global round $t$. Whenever the server receives information from a device, it increments $t$. Upon receiving an update from the server, the device synchronizes by setting $t_m = t$, capturing the latest state. This asynchronous mechanism allows devices to progress without waiting for stragglers, improving throughput and minimizing idle time. The server maintains a copy of the latest local embeddings $\Tilde{h}^t_{m,i}$ for all devices $m \in [M]$ and data samples $\xi_i \in \gD$. Due to the asynchronous nature of the algorithm, these copies may be stale, as they do not always reflect the most up-to-date model parameters of the devices. Let $\Tilde{t}_{m,i}$ denote the device time when the server last updated $\Tilde{h}^t_{m,i}$. The delay at server communication round $t$ can then be expressed as  
\(
\tau_{m}^t = t_m - \Tilde{t}_{m,i},    
\)
where the delayed model parameters are defined as:  
\begin{align}\label{eq:delay_def}
\Tilde{\chi}^t = \{w_1^{t_1 - \tau_1^t}, \dots, w_M^{t_M - \tau_M^t}\}, \quad \Tilde{\vw}^t = \{w_0^t, \Tilde{\chi}^t\}.
\end{align} 

% The delay parameter $\tau_{m}^t$ arises either from the client has locally updated its parameters since its last communication, or the server multiple global updates based on previous forward information.

\textbf{Local Embedding Finite Differences.} For each global iteration $t$, a device $m$ is activated, and it samples a mini-batch $\gB_{m}^{t_m}\in\gD$ and the corresponding IDs $\gI_{m}^{t_m}$. To approximate gradients via zeroth-order finite differences, device $m$ computes two perturbed local embeddings for the mini-batch, 
\begin{align*}
&\{h_{m,i}^{t_m+}\}_{i\in \gI_m^{t_m}}=\{h(w_m^{t_m}+\lambda_m \vu_{m}^{t_m};\xi_{m,i})\}_{i\in \gI_m^{t_m}},\\
&\{h_{m,i}^{t_m-}\}_{i\in \gI_m^{t_m}}=\{h(w_m^{t_m}-\lambda_m \vu_{m}^{t_m};\xi_{m,i})\}_{i\in \gI_m^{t_m}}\numbereq\label{eq:perturb}
\end{align*}
where $\vu_{m}^{t_m}$ is sampled uniformly at random from the Euclidean sphere $\sqrt{d_m}\sS^{d_m-1}$, and $\lambda_m$ is a smoothing parameter that controls the step size of perturbation. 
These two perturbed embeddings serve as “positive” and “negative” perturbations of its local parameters, which are forwarded to the server for further computation.

\begin{algorithm}[t]
\DontPrintSemicolon
   \caption{{\tt DPZV}: Differentially Private Zeroth-Order Vertical Federated Learning}
   \label{alg:DPZV_asyn}
   \SetKwInOut{Input}{Input}
   \SetKwInOut{Output}{Output}
   \Input{Data $\gD$, batch size $B$, learning rate $\eta_m$, total iteration $T$, smoothing parameter $\lambda_m>0$, clipping threshold $C>0$, privacy parameter $\sigma_{dp}$}
    \Output{Parameter $w_0, w_m$ for all parties $m\in[M]$}
    Initialize $w_0,w_m$ and set $t,t_m=0$ for all parties\;
    \For{$t = 1,\ldots,T$}{
    Sample ready-to-update device $m\in \{1, \ldots, M\}$.\;
    device $m$ samples a mini-batch $\gB_m^{t_m}$ with  corresponding IDs $\gI_m^{t_m}$.\;
    device $m$ computes perturbed local embeddings:
    \[\{h_{m,i}^{t_m+}\}_{i\in \gI_m^{t_m}}=\{h(w_m^{t_m}+\lambda_m \vu_{m}^{t_m};\xi_{m,i})\}_{i\in \gI_m^{t_m}},\]\[\{h_{m,i}^{t_m-}\}_{i\in \gI_m^{t_m}}=\{h(w_m^{t_m}-\lambda_m \vu_{m}^{t_m};\xi_{m,i})\}_{i\in \gI_m^{t_m}}\]
    \colorbox{lightred}{\textbf{(Forward)} Device transmits embeddings to server.}\;
    Server computes $\Delta_{m}^{t}$, and updates via ZO Optimization:
    \(\delta_{m,i}^{t,t_m} = \frac{\Tilde{f}(w_0,h_{m,i}^{t_m+};y_i)-\Tilde{f}(w_0,h_{m,i}^{t_m-};y_i)}{\lambda_m},\)
    \(\Delta_{m}^{t} = \frac{1}{B}\sum_{i \in \gI_{m}^{t_m}} \mathrm{clip}_C(\delta_{m,i}^{t,t_m})+z_{m}^{t}.\)
    % Server updates global server with $\Delta_{m}^{t}$ via ZO Optimization.\;

    \colorbox{lightgreen}{\textbf{(Backward)} Device receives $\Delta_{m}^{t}$. }\;
     
     Device performs local update:
     \[\vw_m \leftarrow \vw_m -\eta_m\Delta_{m}^{t}\vu_{m}^{t_m}.\]
     }
     Device update local time stamp $t_m = t$.\;
\end{algorithm}
\subsection{Server Update and Backward Communication}
The server updates the global model and transmit global updates to the device. We define the server-device communication as \textit{backward communication}.

\textbf{Server Side ZO Computation.} 
% For each embedding  received from client $m$. Once the 
After the server receives the local embeddings $h_{m,i}^{t_m}$ from device $m$, it updates its embeddings copy and do computation. For each embedding pair $\{h_{m,i}^{t_m+}, h_{m,i}^{t_m-}\}$, it computes the difference in loss function $\gL$ caused by perturbation, divided by the smooth parameter $\lambda_m$. Specifically, we define\footnote{We slightly abuse the notation, and define $\Tilde{f}(w_0,h_{m,i}^{t_m\pm};y_i)=\gL(w_0, \Tilde{h}^t_{1,i}, \dots,h_{m,i}^{t_m\pm},\dots, \Tilde{h}^t_{M,i};y_i)$.
where we treat $\Tilde{f}$ as a function of the server parameter $w_0$ and the perturbed local embeddings.}:
\begin{align}\label{eq:delta}
\delta_{m,i}^{t,t_m} = \frac{\Tilde{f}(w_0,h_{m,i}^{t_m+};y_i)-\Tilde{f}(w_0,h_{m,i}^{t_m-};y_i)}{\lambda_m},
\end{align}
\textbf{Scalar DP Noise Injection.} To ensure controllable privacy, the server then clips each $\delta_{m,i}^{t}$ by a threshold $C$ to bound sensitivity:
\(
\mathrm{clip}_C(\delta_{m,i}^{t,t_m}) = \min \{\delta_{m,i}^{t,t_m},C\}.
\)
It then samples noise $z_{m}^{t_m}$ from a Gaussian distribution $\mathcal{N}(0,\sigma_{dp}^2)$. This noise is added to the mean of the per-sample-clipped updates, yielding a differentially private gradient-like quantity:
\begin{align}\label{eq:Delta}
\Delta_{m}^{t} = \frac{1}{B}\sum_{i \in \gI_{m}^{t_m}} \mathrm{clip}_C(\delta_{m,i}^{t,t_m})+z_{m}^{t}.
\end{align}

The server then performs a {backward communication} and sends $\Delta_{m}^{t}$ to device $m$. The server then updates its global model through two possible operations: $i)$ ZO Optimization, $\vw_0 \leftarrow \vw_0 - \eta_0 g(w_0;\xi_i)$ (defined in \eqref{eq:zero}); or $ii)$ stochastic gradient descent (SGD), $\vw_0 \leftarrow \vw_0 - \eta_0 \nabla_{w_0} F(\vw;\xi_i)$, 

depending on the constraints on computation resources. The adopted ZO update methodology is explained in Sec \ref{sec:background_and_motivation}.

On the device side, upon receiving $\Delta_{m}^{t}$, device $m$ updates its local parameter $\vw_m$ with learning rate $\eta_m$:
\begin{equation}
\vw_m \leftarrow \vw_m -\eta_m\Delta_{m}^{t}\vu_{m}^{t_m},
    \label{eq:local_gd}
\end{equation}
where $\vu_{m}^{t_m}$ is the same vector used for local perturbation. Our method is well-suited for resource-constrained environments, such as edge devices with limited VRAM or computation power, while still maintaining robust performance and scalability in VFL settings. We summarize the pipeline above in Algorithm~\ref{alg:DPZV_asyn}. 

\section{Convergence Analysis}
\label{sec:convergence}
In this section, we provide the convergence analysis for {\tt DPZV}. For brevity, we define the following notations: $F^t = F(\vw^t):=F(\vw^t;\gD,\mY)$, and $f(\vw;\xi_i)$ as defined in \eqref{eq:f}.
We make the following standard assumptions \footnote{Assumption \ref{assum:Lip} and \ref{assum:bound} are standard in VFL and ZO literature \cite{wang2024unified, castiglia2023flexible}. We follow \cite{zhang2024dpzero} to make the $\ell$-Lipschitz assumption in order to bound the probability of clipping. Assumption \ref{assum:ind_part} is common when dealing with asynchronous participation \cite{chen2020vafl}, and can be satisfied when the activations of devices follow independent Poisson processes.}:

\begin{assumption}[Properties of loss function]\label{assum:Lip}
    The VFL objective function $F$ is bounded from below, the function $f(\vw;\xi_i)$ is $\ell$-Lipschitz continuous and $L$-Smooth for every $\xi_i\in\gD$. 
\end{assumption}

\begin{assumption}[System boundedness]\label{assum:bound}
The following system dynamics are bounded:
1) \textit{Stochastic Noise:}
The variance of the stochastic first order gradient is upper bounded in expectation:
    \(\E\left[\norm{\nabla_{\vw}f(\vw;\xi)-\nabla_{\vw}F(\vw)}^2\right]\le \sigma_s^2.\)
2) \textit{Time Delay:}
    The parameter delay $\tau_m^t$ is upper bounded by a constant $\tau$:
    \( \tau^t \leq \tau,\quad \forall m,t.\)
    % , and 
\end{assumption}
% To deal with the asynchronous nature in our algorithm, we also make the following assumption for analyzing the participation of clients:
\begin{assumption}[Independent Participation]\label{assum:ind_part}
    Under an asynchronous update system, the probability of one device participating in one communication round is independent of other devices and satisfies:
    \(\Prob(\text{device }m\text{ uploading}) = q_m .\)
\end{assumption}
\subsection{Convergence Guarantee and Discussion}
We now present the main theorem that provides convergence guarantee for {\tt DPZV}:
% \chaoyue{Why don't put the uniformly bounded delay assumption in the Assumption 5.3? It seems weird to have one assumption in the theorem, and others not}
\begin{theorem}\label{thm:main}
Under assumption \ref{assum:Lip}-\ref{assum:ind_part}, define $\gF=\E[F^0-F^T]$. Denote $q_* = \min_m q_m$, $d_*=\max_m d_m$ where $d_m$ represent the dimension of model parameters on device $m$, let all step sizes satisfy: $\eta_0 = \eta_m = \eta\le \min\{\frac{1}{\sqrt{Td_*}}, \frac{B}{4L(B+8d_0)+8\gamma_1(2d_m+B)}\}$, let the smoothing parameter $\lambda$ satisfy:
\(
    \lambda \le \frac{1}{Ld\sqrt{T}}
\), and let the clipping level $C$ satisfy: $ C\geq \max\{0, \frac{1}{2}L\lambda d - \ell\sqrt{8\log(2\sqrt{2\pi})}\}.$
Then, for any given global iteration $T \geq 1$, we have the following upper bound on the gradient norm:
\begin{align*}
    &\frac{1}{T} \sum_{t=0}^{T-1}\E\left[\norm{\nabla_{\vw}F(\vw^t)}^2\right]
    \\
    \le&   \mathcal{O}\bigg(\frac{\gF\sqrt{d_*}}{\sqrt{T}}  + \frac{d_*}{T} +  \frac{(\sigma_s^2/B + \sigma_{dp}^2)\sqrt{d_*}}{\sqrt{T}}\\
    &+\frac{ C^2 \sqrt{d_*}}{(\exp (C^2)-1)B\sqrt{T}} \bigg),\numbereq\label{eq:main}
\end{align*}
% where $\Xi=2\sqrt{2\pi}\exp(-(2C-L\lambda d)^2/32\ell^2)$, $\gamma_1$ is a constant further defined in Appendix~\ref{appen:lemma}, and $d$ is the dimension of all the trainable parameters in the VFL framework. 
\end{theorem}
% \chaoyue{This proof sketch is not informative. Should describe the main proof logic and techniques}

% We can observe that the convergence rate is
% $\gO(\sqrt{d_*/T})$, .
\textbf{Discussion.}
The first term is influenced by the model's initialization, $F^0$. This term also enjoys the same rate as ZO optimization in the centralized case \cite{ghadimi2013stochastic}. The second term $\mathcal{O}(\frac{d_*}{T})$ is a standard term for ZOO methods based on the usage of ZO estimator updates. The third term captures the impact of various noise sources in the learning system. Here, $\sigma_s^2/B$ represents the noise introduced by stochastic gradients, and $\sigma_{dp}^2$ corresponds to the variance of the injected DP noise. While reducing $\sigma_{dp}^2$ improves utility, it comes at the cost of weakening privacy guarantees. Consequently, this term encapsulates the fundamental trade-off between model performance, computational cost, and privacy budget. The fourth term, quantifies the impact of the gradient clipping operation on convergence. As $C$ increases, increases, this term diminishes, effectively recovering the non-private case in the limit. However, the noise variance $\sigma_{dp}^2$ also scales linearly with $C$ , which can impede overall convergence. This trade-off underscores the importance of carefully tuning the sensitivity level to balance privacy preservation and learning efficiency.

By combining the convergence analysis in Theorem~\ref{thm:main} with the privacy mechanism in Section~\ref{sec:privacy}, we obtain the following corollary, which characterizes the privacy-utility tradeoff of \texttt{DPZV}.

\begin{corollary}
\label{cor:privacy_utility}
Under the assumptions of Theorem~\ref{thm:main}, consider the differentially private mechanism described in Section~\ref{sec:privacy}, where the variance of the injected DP noise is chosen as
$\sigma_{\mathrm{dp}} = \frac{2C\sqrt{T}}{D\mu}.$
If the total number of global iterations is selected as
\begin{equation}
T = \gO\!\left(
\frac{D^{2}\mu^{2}\sqrt{d}}{C^{2}}
\left(
\gF + \frac{\sigma^{2}}{B} + \frac{C^{2}}{e^{C^{2}}-1}
\right)
\right),
\end{equation}
then the convergence bound in Theorem~\ref{thm:main} implies
\begin{equation}
\frac{1}{T}\sum_{t=0}^{T-1}
\mathbb{E}\!\left[\left\|\nabla \mathcal{F}(w^{t})\right\|^{2}\right]
= \gO\!\left(\frac{\sqrt{d}}{D\mu}\right).
\end{equation}
\end{corollary}
\textbf{Discussion.}
Corollary \ref{cor:privacy_utility} shows that, under differential privacy constraints, the proposed zeroth-order method attains the same asymptotic privacy–utility tradeoff as differentially private first-order methods, such as DP-SGD \cite{chen2020understanding}. This result highlights an important contrast with the non-private setting: although zeroth-order optimization generally exhibits slower convergence than first-order methods, the proposed DP noise injection strategy reshapes the dominant error terms in the convergence analysis, allowing zeroth-order methods to achieve comparable asymptotic guarantees.

\subsection{Proof of Theoretical Results}

\label{appen:thm_conv}
In this part, we briefly describe the main steps behind the main convergence theorem \ref{thm:main}. 

The proof of the theorem follows the standard smoothness-
based descent framework for non-convex optimization. However, the convergence is affected by stochastic sampling, zeroth-order perturbations, differential privacy noise, gradient clipping and communication delays, and we need to bound them. The following lemmas solve these questions, and the full details are relegated to Appendix \ref{appen:prelim}.

We first introduce the following lemma, which bound how much the ZO estimator is biased from the ground truth gradient.

\begin{lemma}\label{lem:zero}
Let $g(\vw)$ be the zeroth-order gradient estimator defined as in \eqref{eq:zero_def}, with $f(\vw)$ being the loss function. We define the smoothed function $f_\lambda(\vw) = \E_{\vu}[f(\vw + \lambda \vu)]$, where $v$ is uniformly sampled from the Euclidean ball $\sqrt{d} \mathbb{B}^d = \{\vw \in \mathbb{R}^d \mid \norm{\vw} \leq \sqrt{d}\}$. The following properties hold:
\begin{enumerate}[label=$(\roman*)$]
    \item $f_\lambda(\vw)$ is differentiable and $\E_\vu[g(\vw)] = \nabla f_\lambda(\vw)$.
    \item If $f(\vw)$ is $L$-smooth, we have that
    \begin{align}
        \norm{\nabla f(\vw) - \nabla f_\lambda(\vw)} \le \frac{L}{2} \lambda d^{3/2},\label{eq:grad_diff}
    \end{align}
    \begin{align}
        | f(\vw) -  f_\lambda(\vw)| \le \frac{L}{2} \lambda^2 d,\label{eq:func_diff}
    \end{align}
    and
    \begin{align}
    \E_\vu[\norm{g_\lambda(\vw)}^2] \le 2d \cdot \norm{\nabla f(\vw)}^2 + \frac{L^2}{2} \lambda^2 d^3.\label{eq:zero_grad}
    \end{align}
\end{enumerate}
\end{lemma}
This is the standard result of zeroth-order optimization. The proof of the Lemma is given by \cite{nesterov2017random}. In typical differential privacy algorithm, we need to clip the gradient to control the sensitivity, which is also the case in our algorithm. To control the effect of clipping, the following lemma evaluates the probability that clipping happens, so that we can take expectation w.r.t. the clipping effect in the future. 
\begin{lemma}\label{lem:Q}
    Let $Q$ be the event that clipping happened for a sample $\xi$, $d$ be the model dimension, and $L,\ell$ be the Lipschitz and smooth constant as defined in assumption \ref{assum:Lip}. For $\forall C_0>0$, if the clipping threshold $C$ follows $C\ge C_0+L\lambda d/2$, we have the following upper bound for the probability of clipping:
    \begin{align}
        P = \Prob(Q)\le \underbrace{2\sqrt{2\pi}\exp(-\frac{C_0^2}{8\ell^2})}_{\Xi}
    \end{align}
\end{lemma}

Furthermore, we have the following lemma that bounds the variance of the private zeroth-order gradient estimator, leveraging the auxiliary lemmas, Lemma \ref{lem:zero} and \ref{lem:Q}.
% Finally, we combine the result in Lemma \ref{lem:zero} and \ref{lem:Q} to derive the following lemma that bounds the variance of the private zeroth-order gradient estimator. 
It addresses the effect of stochastic sampling, zeroth-order bias, gradient clipping and DP noise at the same time. 
\begin{lemma}Let $\breve{G}^t_{m}(\vw^t)$ be the Differential Private Zeroth-order Gradient without delay. We can bound the variance of $\breve{G}^t_{m}(\vw^t)$ in expectation, with the expectation taken on random direction $\vu$, DP noise $z$, and clipping event $Q$: 
\begin{align*}
&\Var(\breve{G}^t_{m}(\vw^t))\le\\
&\frac{1-P}{B}(2d_m \E\norm{\nabla_{w_m} F(\vw^t)}^2 + 2d_m\sigma_s^2+ \frac{L^2}{2} \lambda^2 d_m^3 )\\
&+\frac{P}{B} C^2 d_m-\frac{(1-P)^2}{B}\E\norm{\nabla_{w_m} F_\lambda(\vw^t)}^2+\sigma_{dp}^2 d_m
\numbereq\label{eq:G_var}
\end{align*}
\label{lem:var}
\end{lemma}

% \begin{theorem}
% Under Assumption \ref{assum:Lip_full} to \ref{assum:ind_part_full}, and assume the device delay $\tau_m$ is uniformly upper bounded by $\tau$. If we denote $q_* = \min_m q_m$, $d_*=\max_m d_m$, and let $\eta_0 = \eta_m = \eta\le \min\{\frac{1}{\sqrt{Td_*}}, \frac{B}{4L(B+8d_0)+8\gamma_1(2d_m+B)}\}$, $\lambda \le \frac{1}{Ld\sqrt{T}}$, we have the following theorem:
% \begin{align*}
%     &\frac{1}{T} \sum_{t=0}^{T-1}\E\left[\norm{\nabla_{\vw}F(\vw^t)}^2\right]\\
%     \le &\frac{1}{1-\Xi}\left\{\frac{8\sqrt{d_*}}{q_* T^{1/2}}\E\left[F^{0}-F^{T}\right] + \frac{2d_*}{q_* T}+ \frac{8\sqrt{d_*}}{q_*Ld T^{1/2}}\right.\\
%     &+\frac{16 (4L+2\gamma_1)\sqrt{d_*}\sigma_s^2}{q_* BT^{1/2}}+\frac{4 \left((4-B)L+(B+2)\gamma_1\right)\sqrt{d_*}}{q_*B T^{1/2}}\\
%     &\left.+\frac{16 (2L+\gamma_1)\Xi C^2 \sqrt{d_*}}{q_* BT^{1/2}}+ \frac{16 (2L+\gamma_1)\sigma_{dp}^2 d_m}{q_*T^{1/2}\sqrt{d_*}}\right\},\numbereq
% \end{align*}
% where $\Xi = 2\sqrt{2\pi}\exp(-\frac{C_0^2}{8\ell^2}).$
% \end{theorem}

Using the lemma above, we are ready to derive the proof for Theorem \ref{thm:main}.

\noindent\textbf{Proof for Theorem \ref{thm:main}.} We start the proof by using the $\ell$-Lipschitz property in Assumption \ref{assum:Lip}, which is the common practice in non-convex optimization. We then taking expectation w.r.t. the random sampling, clipping event and ZO perturbation, where we utilized Lemma \ref{lem:zero} and \ref{lem:Q}. Finally, we use the result in Lemma \ref{lem:var} to bound the variance terms. The result is the following one-step descent Lemma \ref{lem:main_1}, which quantifies the decrease of the objective function in a single communication round. 

\begin{lemma}[Model Update With Delay]\label{lem:main_1}
     Under Assumption \ref{assum:Lip} to \ref{assum:ind_part}, we have the following lemma:
     \begin{align*}
         &\E\left[F(\vw^{t+1})-F(\vw^{t}) \right] \le-\sum_{m=0}^M\bigg\{\\
         &q_m\eta_m(1-P)\left(\frac{1}{4}-\frac{\eta_m L (B+8d_m)}{2B}\right)\E\norm{\nabla_{w_{m}}F(\vw^t)}^2\bigg\}\\
         &+\sum_{m=0}^{M}q_m\eta_m L \left(\frac{1}{2}+2\eta_m L^2\right)\E\left[\norm{\Tilde{\chi}^t-\chi^t}^2\right]+\gA_1\numbereq,
     \end{align*} where we denote
\begin{align*}
    \gA_1 &= \sum_{m=0}^M q_m\eta_m \frac{L^2}{8}\lambda^2 d_m^3+ \sum_{m=0}^M q_m \eta_m^2 L \left(\frac{4}{B}d_m\sigma_s^2\right.\\
    &\left.+ \frac{(4-B)L^2}{4B} \lambda^2 d_m^3 +\frac{2P}{B} C^2 d_m+2\sigma_{dp}^2 d_m\right) + L\lambda^2 d
\end{align*}
for the convenience of notation.
\end{lemma}

% In the proof of Lemma \ref{lem:main_1}, 
% \cy{It is not necessary to talk about the proof sketch of this lemma v.9. }
Since this analysis does not account for the effect of delayed communication, we subsequently address it in Lemma~\ref{lem:main_2}. The detailed proof of these two lemmas can be found in Appendix \ref{appen:lemma}.
We account for asynchronous device participation and bound the communication delay term $\E\norm{\Tilde{\chi}^t-\chi^t}^2$ by introducing the following Lyapunov function that captures the discrepancy between delayed and current model parameters: 
\begin{align}\label{eq:lyap}
    V^t = F(\vw^t)+\sum_{i=1}^{\tau}\gamma_i\norm{\chi^{t+1-i}-\chi^{t-i}}^2
\end{align}
    with $\gamma_i$ to be determined later. 

Utilizing the Lyapunov function, we can derive Lemma \ref{lem:main_2}, which quantifies the descent of the Lyapunov function in each communication round. This lemma successfully quantifies the effect of communication delay, while the descent of Lyapunov function provides an upper bound for the descent of the objective function.
\begin{lemma}
\label{lem:main_2}
    Under Assumption \ref{assum:Lip}-\ref{assum:ind_part}, and assume the device delay $\tau_m$ is uniformly upper bounded by $\tau$, we have the following lemma:
    \begin{align*}
        \E\left[V^{t+1}-V^{t}\right]&\le-\frac{1-P}{8}\min_m\{q_m\eta_m\} \E\left[\norm{\nabla_{\vw}F(\vw^t)}^2\right]\\
        &+\gA_1 + \gA_2\numbereq\label{eq:interm},
    \end{align*}
    where we denote 
\begin{align*}
    \gA_2
    =&\sum_{m=1}^{M}q_m\eta_m^2\gamma_1\bigg(\frac{L^2}{4}\lambda^2 d_m^3+\frac{2}{B}d_m\sigma_s^2 + \\
    &\frac{L^2}{2B} \lambda^2 d_m^3 +\frac{P}{B} C^2 d_m+\sigma_{dp}^2 d_m\bigg)
\end{align*}
for the ease of notation.
\end{lemma}

By combining the descent inequality derived from Lemma 
\ref{lem:main_2} with appropriate choices of step size and smoothing parameters, we obtain a bound on the average gradient norm over $T$ iterations, which establishes the convergence rate as stated in Theorem \ref{thm:main}. 
We start from the result of Lemma \ref{lem:main_2}. 

    Re-arrange the terms in \eqref{eq:interm}, and take average over $0,1,\dots, T-1$:
\begin{align*}
    &\frac{1-P}{8T}\min_m\{q_m\eta_m\} \sum_{t=0}^{T-1}\E\left[\norm{\nabla_{\vw}F(\vw^t)}^2\right]\\
    \le& \frac{1}{T}\E\left[V^{0}-V^{T}\right] + \gA_1 + \gA_2\\
    \le& \frac{1}{T}\E\left[F^{0}-F^{T}\right] + \gA_1 + \gA_2,
\end{align*}
where the second inequality follows from the definition of $V^t$ in \eqref{eq:lyap}.

Dividing $\alpha =\frac{1}{8}\min_m\{q_m\eta_m\}$ from both sides, and plugging in $\gA_1$ and $\gA_2$, we have:
\begin{align*}
    &\frac{1-P}{T} \sum_{t=0}^{T-1}\E\left[\norm{\nabla_{\vw}F(\vw^t)}^2\right]\\
    \le &\frac{1}{\alpha T}\E\left[F^{0}-F^{T}\right] + \frac{1}{\alpha}\sum_{m=0}^M q_m\eta_m \frac{L^2}{8}\lambda^2 d_m^3\\
    &+ \frac{1}{\alpha}\sum_{m=0}^M q_m \eta_m^2 L \left(\frac{4}{B}d_m\sigma_s^2\right. \\
    &\left.+ \frac{(4-B)L^2}{4B} \lambda^2 d_m^3 +\frac{2P}{B} C^2 d_m+2\sigma_{dp}^2 d_m\right)\\
    &+ \frac{1}{\alpha}L\lambda^2 d +\frac{1}{\alpha}\sum_{m=1}^{M}q_m\eta_m^2\gamma_1\left(\frac{L^2}{4}\lambda^2 d_m^3+\frac{2}{B}d_m\sigma_s^2\right. \\
    &\left.+ \frac{L^2}{2B} \lambda^2 d_m^3 +\frac{P}{B} C^2 d_m+\sigma_{dp}^2 d_m\right)
\end{align*}
For the simplicity of analysis, we let $\eta_0 = \eta_m = \eta$, $q_* = \min_m q_m$, then $\alpha =\frac{\eta}{8}q_*$. Let $d_*=\max_m d_c$, and $\lambda \le \frac{1}{Ld\sqrt{T}}$. Thus, 
\begin{align*}
    &\frac{1-P}{T} \sum_{t=0}^{T-1}\E\left[\norm{\nabla_{\vw}F(\vw^t)}^2\right]\\
    \le &\frac{8}{q_*\eta T}\E\left[F^{0}-F^{T}\right] + \frac{2d_*^3/d^2}{q_* T}+ \frac{16\eta L}{q_*}  \left(\frac{4}{B}d_*\sigma_s^2 \right.\\
    &\left.+ \frac{(4-B)d_*^3/d^2}{4BT}  +\frac{2P}{B} C^2 d_*+2\sigma_{dp}^2 d_*\right)\\
    +& \frac{8}{q_*Ld\eta T} +\frac{16\eta \gamma_1}{q_*}\left(\frac{2}{B}d_*\sigma_s^2 + \frac{(B+2)d_*^3/d^2}{4BT}\right. \\
    &\left.+\frac{P}{B} C^2 d_*+\sigma_{dp}^2 d_m\right)\\
    \le & \frac{8}{q_*\eta T}\E\left[F^{0}-F^{T}\right] + \frac{2d_*}{q_* T}+ \frac{8}{q_*Ld\eta T}\\
    &+\frac{16\eta (4L+2\gamma_1)d_*\sigma_s^2}{q_* B}+\frac{4\eta \left((4-B)L+(B+2)\gamma_1\right)d_*}{q_*B T}\\
    +&\frac{16\eta (2L+\gamma_1)P C^2 d_*}{q_* B}+ \frac{16\eta (2L+\gamma_1)\sigma_{dp}^2 d_*}{q_*}
\end{align*}
where in the last step, we use the fact that $d > d_*$.\\
If we choose $\eta = \frac{1}{\sqrt{Td_*}}$ and use Lemma \ref{lem:Q}., we can get the convergence rate:
\begin{align*}
    &\frac{1}{T} \sum_{t=0}^{T-1}\E\left[\norm{\nabla_{\vw}F(\vw^t)}^2\right]\\
    % \le &\frac{1}{1-P}\left\{\frac{8\sqrt{d_*}}{q_* T^{1/2}}\E\left[F^{0}-F^{T}\right] + \frac{2d_*}{q_* T}+ \frac{8\sqrt{d_*}}{q_*Ld T^{1/2}}\right.\\
    % &+\frac{16 (4L+2\gamma_1)\sqrt{d_*}\sigma_s^2}{q_* BT^{1/2}}+\frac{4 \left((4-B)L+(B+2)\gamma_1\right)\sqrt{d_*}}{q_*B T^{1/2}}\\
    % &\left.+\frac{16 (2L+\gamma_1)P C^2 \sqrt{d_*}}{q_* BT^{1/2}}+ \frac{16 (2L+\gamma_1)\sigma_{dp}^2 
    % \sqrt{d_*}}{q_*T^{1/2}}\right\}\\
    \le &\frac{1}{1-\Xi}\left\{\frac{8\sqrt{d_*}}{q_* T^{1/2}}\E\left[F^{0}-F^{T}\right] + \frac{2d_*}{q_* T}+ \frac{8\sqrt{d_*}}{q_*Ld T^{1/2}}\right.\\
    &+\frac{16 (4L+2\gamma_1)\sqrt{d_*}\sigma_s^2}{q_* BT^{1/2}}
    +\frac{4 \left((4-B)L+(B+2)\gamma_1\right)\sqrt{d_*}}{q_*B T^{1/2}}\\
    &\left.+\frac{16 (2L+\gamma_1)\Xi C^2 \sqrt{d_*}}{q_* BT^{1/2}}+ \frac{16 (2L+\gamma_1)\sigma_{dp}^2 
    \sqrt{d_*}}{q_*T^{1/2}}\right\}
\end{align*}

We thus conclude that the convergence rate is 
\[\gO\left(\frac{d_*^{1/2}\E\left[F^{0}-F^{T}\right]+d_*^{1/2}\sigma_s^2+d_*^{1/2}\sigma_{dp}^2}{T^{1/2}}\right),\]
where constant factors are absorbed into the $O(\cdot)$-notation for simplicity.
% s before terms have been omitted for simplicity.

\section{Privacy Analysis}
\label{sec:privacy}

% \subsection{Threat Model and Assumptions}
In this section, we talk about how our algorithm protects privacy under threats of label inference attacks and feature inference attacks.
\subsection{Threat Model and Risk Evaluation}
\textbf{Threat Model. }
We consider the \textit{honest-but-curious} (HBC) threat model, where parties correctly follow the training protocol but may attempt to infer sensitive information from exchanged messages. In paticular, we assume that the server, which holds the labels and performs loss computation, is trusted, while devices, who holds partial data at local, may attempt to infer private information from intermediate results exchanged during training~\cite{papernot2018sok}. 

This assumption is motivated by typical real-world deployment scenarios of VFL, such as finance, healthcare, and online advertising,  where the server corresponds to a centralized service provider (e.g., a bank, hospital, or platform operator) that owns the prediction task, maintains the training infrastructure, and is subject to strict regulatory and contractual obligations. In contrast, devices are often data holders who contribute feature sets but do not control the overall learning objective. Similar trust assumptions have been widely adopted in prior VFL works \cite{cheng2021secureboost,wu2020privacy,xu2021fedv,ye2025vertical}, where attacking techniques and privacy protection is primarily designed for feature-holding parties.

\textbf{Training-Time Privacy Risks in VFL.}
Under the HBC model, privacy risks arise from information exchanged during training. Among existing attacks, \textbf{1)} label inference attacks, which aim to recover sensitive labels by exploiting gradients or other backward information received during training, and \textbf{2)} feature inference attacks, which attempt to infer private feature values of a party, are recognized as the two dominant threats in VFL. Existing studies have shown that these two attack categories capture the primary privacy risks in VFL settings. Therefore, in this work, we focus exclusively on label inference and feature inference attacks during training.

\textbf{Implications for Privacy Mechanism Design.}
Given this threat model, we argue that privacy protection is unnecessary for forward feature transmissions, which are sent from passive parties to the trusted server (as in VAFL\cite{chen2020vafl}).

Since the server is assumed to be non-adversarial, protecting forward features does not mitigate the considered privacy risks and would unnecessarily distort the information required for accurate loss computation. In contrast, backward information, such as gradients or gradient-related signals, is transmitted from the server to potentially curious passive parties and has been shown to be the primary carrier of both label information and sensitive feature correlations. Consequently, backward messages constitute the main attack surface for both label inference and feature inference attacks under the HBC model. Therefore, we apply privacy mechanisms \textbf{exclusively to the backward information}, which effectively mitigates the dominant training-time privacy risks while preserving model utility.

\subsection{Theoretical Differential Privacy Guarantee}
In this part, we give rigorous proof for the differential privacy guarantee in Theorem \ref{thm:dp}. 

We first introduce the definition of Gaussian differential privacy (GDP)~\cite{dong2022gaussian} which will be useful in the proof for Theorem \ref{thm:dp}. Compared with traditional DP defined in \eqref{eq:DP}, this notion of privacy provides a much tighter composition theorem, thus requiring less noise to be injected each round for a given privacy level.
\begin{definition}[Gaussian Differential Privacy]\label{def:GDP}
    Let $G_\mu:=T(\gN(0,1),\gN(\mu,1))$ for $\mu\ge0$, where the trade-off function $T(P,Q):[0,1]\rightarrow [0,1]$ is defined as $T(P,Q)(\alpha)=\inf(\beta_\phi: \alpha_\phi<\alpha)$.
    A mechanism $M$ is said to satisfy $\mu$-Gaussian Differential Privacy if it satisfies
    \[T(M(X),M(X'))\ge G_\mu\] 
    For all neighboring dataset $X,X'\in \gX^n$.
\end{definition}
We construct our DP algorithm based on the Gaussian Mechanism. Specifically, Gaussian mechanism ensures that, by injecting a designed Gaussian noise into a random algorithm, the algorithm is differentially private, and it provides a guidance to how large the noise should be for reaching a specific privacy level. Before presenting our results, we first present several existing results from prior works. We begin with the Gaussian mechanism under GDP, which is summarized in the following result:
\begin{result}[Gaussian Mechanism for Gaussian Differential Privacy~\cite{dong2022gaussian}]\label{thm:gm}
    Define the Gaussian mechanism that operates on a statistic $\theta$ as $M(\theta)=\theta(X)+\sigma$, where $\sigma\sim\gN(0,r^2C_{\theta}^2/\mu^2)$, $r$ is the sample rate for a single datum, and $C_{\theta}$ is the $L_2$ sensitivity of $X$. Then, $M$ is $\mu$-GDP.
\end{result}

The repeated application of a Gaussian mechanism is known as composition. In this setting, privacy loss accumulates over iterations: adding Gaussian noise multiple times leads to a weaker aggregate privacy guarantee than a single noise injection. Accurately quantifying the resulting privacy level of iterative Gaussian mechanisms, which is exactly the scenario encountered in neural network training, necessitates the following composition theorem.
\begin{result}[Composition of Gaussian Differential Privacy~\cite{dong2022gaussian}]\label{thm:composition}
The $T$-fold composition of $\mu$-GDP mechanisms is $\sqrt{T}\mu$-GDP
\end{result}

Due to differences in their respective notational frameworks, it is not straightforward to directly compare the privacy guarantees of Gaussian DP and the conventional $(\epsilon,\delta)$-DP, which hinders fair comparison with baseline methods. To address this issue and present our results in a commonly adopted $(\epsilon,\delta)$-DP form, we leverage the following lossless conversion theorem to translate Gaussian DP guarantees into equivalent:
\begin{result}[Conversion from Gaussian Differential Privacy to $(\epsilon, \delta)$-Differential Privacy~\cite{dong2022gaussian}]\label{thm:conversion}
    A mechanism is $\mu$-GDP iff it is $(\epsilon, \delta(\epsilon))$-DP for all $\epsilon\ge0$, where
    \begin{align}
        \delta(\epsilon)=\Phi(-\frac{\epsilon}{\mu}+\frac{\mu}{2})-e^\epsilon\Phi(-\frac{\epsilon}{\mu}-\frac{\mu}{2})\label{eq:conversion}
    \end{align}
\end{result}
% Again, the proof of Theorem \ref{thm:gm}-\ref{thm:conversion} can be found in \cite{dong2022gaussian}.

Using the above results, we are now ready to prove the following theorem, which establishes that our algorithm is differentially private and quantifies its privacy level.
\begin{theorem}
\label{thm:dp}
Under assumption \ref{assum:Lip}-\ref{assum:ind_part}, suppose the privacy parameter $\sigma_{dp}$ is
\(\sigma_{dp} = \frac{ 2C\sqrt{T}}{D\mu}, \)
where $D$ denotes the volume of the dataset, $T$ defines total iterations, and $\mu>0$ controls the privacy level.
The training process of Alg \ref{alg:DPZV_asyn} is seen to be $(\epsilon,\delta(\epsilon))$-differential private for $\forall \epsilon>0$, where 
\begin{align}\label{eq:dp_delta}
\delta(\epsilon)=\Phi(-\frac{\epsilon}{\mu}+\frac{\mu}{2})-e^\epsilon\Phi(-\frac{\epsilon}{\mu}-\frac{\mu}{2})\end{align}
\end{theorem}
\begin{proof}
    First recall the definition of $\Delta_{m}^{t}$ defined in \eqref{eq:Delta}:
    \[\Delta_{m}^{t} = \frac{1}{B}\sum_{i \in \gI_{m}^{t_m}} \mathrm{clip}_C(\delta_{m,i}^{t,t_m})+z_{m}^{t}\]
    For a pair of neighboring dataset $X,X'$ differing in only one entry of data, the $L_2$ sensitivity $C_{L}$ of $\frac{1}{B}\sum_{i \in \gI_{m}^{t_m}} \mathrm{clip}_C(\delta_{m,i}^{t,t_m})$ follows by
    \begin{align*}
       C_L =  \norm{\frac{1}{B}\sum_{i \in \gI_{m}^{t_m}} \mathrm{clip}_C(\delta_{m,i}^{t,t_m})}_2
        \le\frac{1}{B}\norm{\mathrm{clip}_C(\delta_{m,i}^{t,t_m})}_2
        \le\frac{2C}{B}
    \end{align*}
    The sample rate $r$ of a single data is seen to be the batch size $B$ divided by the size of the dataset $D$:
    \[r = \frac{B}{D}\]
    Note that in Theorem \ref{thm:dp}, the standard deviation of $z_m^t$ is given by 
    \begin{align*}
        \sigma_{dp} = \frac{ 2C\sqrt{T}}{D\mu}=\frac{ (B/D) (2C/B) \sqrt{T}}{\mu}=\frac{rC_L}{(\mu/\sqrt{T})}
    \end{align*}
    By Result \ref{thm:gm}, the mechanism conducted in \eqref{eq:Delta} satisfies $(\mu/\sqrt{T})$-Gaussian Differential Privacy. Further applying the composition of GDP in Result \ref{thm:composition}, and by the post-processing~\cite{dwork2014algorithmic} of differential privacy, we have that the whole training process of Algorithm \ref{alg:DPZV_asyn} is $\mu$-GDP. We complete the proof by converting $\mu$-GDP to $(\epsilon,\delta)$-DP according to \eqref{eq:conversion}.
\end{proof}
\begin{figure*}[t!]
\centering
        \includegraphics[width=\textwidth]{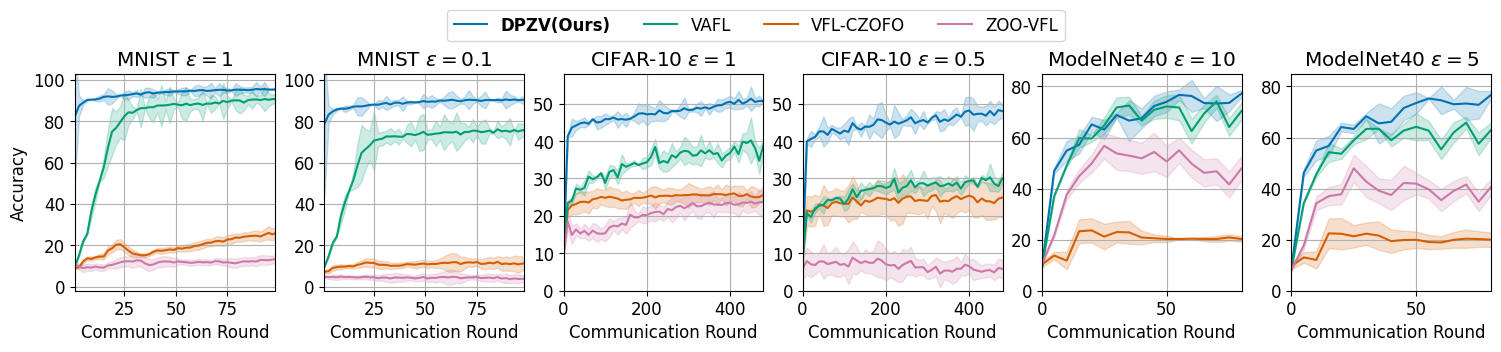}
        \vspace{-1mm}
    \caption{Test Accuracy of VFL Methods on image classification tasks under DP constraints. $\delta$ is set to $1\times 10^{-3}$. 
{\tt DPZV} outperforms first-order VFL methods on two datasets and surpasses all other ZO-based methods across all three datasets, showing both a higher accuracy and a faster convergence rate.
    % We attribute this edge to the computation efficiency of ZO optimization, which significantly reduces model delay.
    }
    \label{fig:main}
\end{figure*}

\begin{figure}[t!]
\centering
    \includegraphics[width=0.5\textwidth]{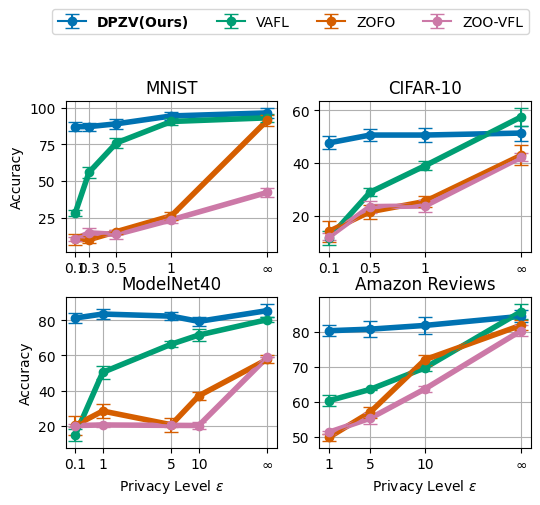}  % Replace with your image file
        \vspace{-3mm}
    \caption{Privacy-Accuracy tradeoff across different datasets and algorithms. We use a constant level of $\delta=1\times 10^{-3}$ and vary $\epsilon$ to simulate different privacy levels. Our algorithm consistently outperforms baselines under tight privacy budget, showing a slower decay in performance than baselines as $\epsilon$ decreases.}
    \label{fig:PA_TO}
\end{figure}
\textbf{Discussion.}
The theorem provides privacy guarantee under the ``honest-but-curious" threat model, where one or a few malicious devices try to do inference attacks by collecting information of the system. By providing only the differentially private ZO information $\Delta_{m}^{t}$, the algorithm protects for labels~\cite{fu2022label} and per-record influence, because the attacker cannot differentiate a single datum in the dataset $\gD$ and label set $\mY$. 
The differential privacy parameter $\mu$ is related to $(\epsilon, \delta)$ through the relationship defined in \eqref{eq:dp_delta}. Given any two of these parameters, the third can be determined by solving \eqref{eq:dp_delta}, allowing full flexibility in controlling the privacy level.

\section{Experiments}
\label{sec:exp}

\subsection{Dataset and Baselines}
We consider four datasets: 
image dataset 
MNIST~\cite{deng2012mnist}, CIFAR-10~\cite{krizhevsky2009learning}, semantic dataset Amazon Review Polarity~\cite{mcauley2013hidden}, and multi-view dataset ModelNet40~\cite{wu20153d}. For each dataset, we conduct a grid search on learning rates and other hyperparameters. We run 100 epochs on each method and select the best validation model. We run each algorithm under three random seeds and compute the sample variance for the generosity of our result. Additional information on the selection of dataset, data processing and model architecture can be found in Appendix \ref{appen:dataset}.

We compare our algorithm against several SotA VFL methods: 1) {\tt VAFL}~\cite{chen2020vafl} 2) {\tt ZOO-VFL}~\cite{zhang2021desirable}, 3) {\tt VFL-CZOFO}~\cite{wang2024unified}. All methods assume that the server holds the labels, and concatenates the embeddings of devices as the input of the server. {\tt VAFL} updates its model through first-order optimization in an asynchronous manner, and achieves DP by adding random noise to the output of each local embedding. We use {\tt VAFL} as a first-order baseline of VFL method. Contrary to {\tt VAFL}, {\tt ZOO-VFL} and {\tt VFL-CZOFO} both adopt ZO optimization in their training procedure. {\tt ZOO-VFL} substitutes the first order optimization by ZO optimization in common VFL methods. {\tt VFL-CZOFO} uses a cascade hybrid optimization method that computes the intermediate gradient via ZOO, while keeping the back propagation on both server and device. To enforce DP and compare all methods fairly, we follow the approach in {\tt VAFL}~\cite{chen2020vafl}, applying the same DP mechanism to both {\tt VFL-CZOFO} and {\tt ZOO-VFL}, which involves clipping the embeddings and adding calibrated vector noise.
% Our empirical results suggest that our {\tt DPZV} method outperforms the aforementioned methods, while requiring less communication overhead and memory footprint, due to the memory-efficient ZOO mechanism and asymmetric communication update design.
In this paper, we only focus on the same update behavior as {\tt VAFL} which updates server and device once for every communication round. In addition, model delay has been manually adjusted for all methods based on the per-batch computation time on a single device to simulate device heterogeneity and ensure a fair comparison.
\subsection{Analysis of Results}
\textbf{Accelerated convergence of {\tt DPZV} with privacy guarantees.}
Figure~\ref{fig:main} presents the performance evaluation of {\tt DPZV} against all baselines on both image classification and language tasks under certain DP levels, where two different privacy budgets $\epsilon$ are evaluated for each dataset. On MNIST and CIFAR-10 with strict privacy constraints ($\epsilon = 1$ and $\epsilon = 0.5$, with $\delta = 1 \times 10^{-3}$), {\tt DPZV} consistently achieves the highest test accuracy and faster convergence across communication rounds.  As the privacy budget becomes tighter, first-order baseline {\tt VAFL} exhibits a pronounced accuracy drop due to the injection of high-dimensional noise, while the zeroth-order baseline {\tt ZOO-VFL} also suffers from severe performance degradation and unstable training dynamics caused by the combined effect of zeroth-order bias and DP noise.In contrast, {\tt DPZV} remains comparatively robust to increasingly stringent privacy constraints, demonstrating a significantly improved privacy--accuracy tradeoff.

On the more challenging ModelNet40 task, where training is more challenging due to larger models and a higher degree of device heterogeneity, we moderately relax the privacy budget ($\epsilon = 10$ and $\epsilon = 5$), and observe that {\tt DPZV} still maintains a competitive advantage over other ZO-based methods while achieving accuracy comparable to the first-order baseline {\tt VAFL}. The scalar noise injected in {\tt DPZV} effectively constrains the noise magnitude, mitigating the instability typically introduced by noisy zeroth-order gradient estimators. This leads to more stable training dynamics than other ZO based methods. Moreover, in contrast to the high-dimensional vector noise used in first-order baseline {\tt VAFL}, scalar noise incurs significantly less optimization distortion, enabling {\tt DPZV} to achieve a more favorable privacy--accuracy tradeoff and, consequently, fewer communication rounds to reach a target accuracy.

 \begin{figure}[t]
% \vskip 0.2in
\begin{center}
\centerline{\includegraphics[width=0.5\textwidth]{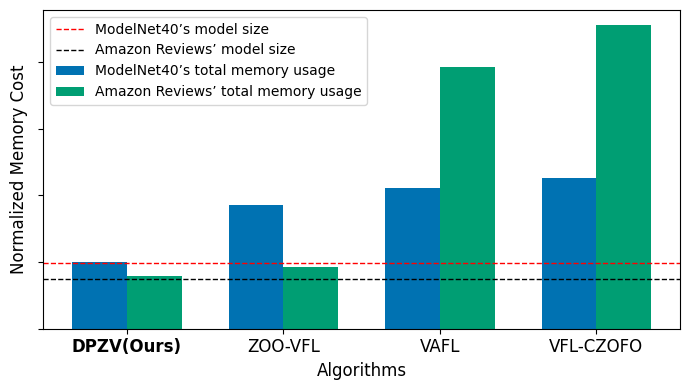}}
\caption{Normalized memory cost in training for each method. {\tt DPZV} requires the smallest memory allocation in both datasets, almost the same as model memory itself. This shows the memory efficiency of {\tt DPZV}, allowing superior performance on large-scale neural networks.}
\label{fig:mem_cost}
\end{center}
\end{figure}

\textbf{{\tt DPZV} elevates privacy-utility tradeoff.}
Figure~\ref{fig:PA_TO} presents the accuracy-privacy tradeoff across four benchmark datasets: MNIST, CIFAR-10, ModelNet40, and Amazon Reviews. We evaluate the robustness of our proposed method \texttt{DPZV} under various privacy budgets $\epsilon$, while fixing the failure probability $\delta=1\times 10^{-3}$. 
% We compare its performance against three baselines: \texttt{VAFL}, \texttt{ZOFO}, and \texttt{ZOO-VFL}. 
Across all tasks, \texttt{DPZV} consistently achieves the highest accuracy under tight privacy regimes ($\epsilon \leq 1$), indicating its ability to maintain model utility even with stringent differential privacy constraints. Notably, on MNIST and ModelNet40, \texttt{DPZV} shows only a marginal drop in accuracy as $\epsilon$ decreases, while the competing methods suffer substantial degradation. For instance, on CIFAR-10 at $\epsilon = 0.1$, \texttt{DPZV} maintains over 90\% accuracy, whereas \texttt{VFL-CZOFO} and \texttt{ZOO-VFL} fall below 30\%.
Similar trends are observed on Amazon Reviews, where \texttt{DPZV} consistently outperforms baselines under strict privacy, especially at $\epsilon = 1$ and $\epsilon = 5$. These results validate the effectiveness of our scalar-noise-based DP mechanism in balancing utility and privacy, and \textit{highlight its advantage in real-world privacy-sensitive federated learning scenarios}.

\begin{table}[t!]
\caption{Total communication cost incurred by each method until reaching a fixed target accuracy for each dataset. The target accuracy for each dataset is 90\%, 40\%, 50\%, and 80\%}
\label{tab:commu_energy}
\resizebox{\linewidth}{!}{
\begin{tabular}{lcccc}
\toprule
Dataset & DPZV(Ours) & VAFL & CZOfirst-order & ZOO-VFL \\
\midrule
MNIST(MB)
& 576.0 \com{31.7} 
& 1209.6 \com{98.9} 
& 4492.8 \com{261.7} 
& $\infty$
\\
CIFAR-10(GB)
& 2.74 \com{0.16} 
& 2.26 \com{0.12} 
& 20.02 \com{1.94} 
& 23.90 \com{3.25} 
\\
ModelNet40(MB)
& 302.4 \com{17.7} 
& 423.36 \com{22.0} 
& 604.8 \com{28.2} 
& 5927.04 \com{328.5} 
\\
Amazon Reviews(GB)
& 13.27 \com{0.78} 
& 16.59 \com{0.86} 
& 36.59 \com{3.70} 
& 69.67 \com{3.86} 
\\
\bottomrule
\end{tabular}
}
\end{table}

\textbf{{\tt DPZV} mitigates memory overhead.} Figure \ref{fig:mem_cost} 
compares the GPU memory consumption, with memory values normalized for readability. We compare the memory cost on larger models, where we use ResNet for image classification and BERT for sequence classification. We record the highest memory peak to show the total required memory for each method on training large models. 
We observe that {\tt DPZV} requires memory approximately equal to the model size, whereas the first-order method {\tt VAFL} demands more than twice the model size. Compared to {\tt ZOO-VFL}, {\tt DPZV} achieves further memory savings by leveraging {\tt MeZO}. 
% This advantage becomes more pronounced for larger models.
\textit{These results highlight the scalability of {\tt DPZV}, making it well-suited for deploying large pretrained language models in VFL scenarios}.
% In Figure ~\ref{fig:PA_TO} we plot the privacy : (1) no DP noise injected, (2) noise injected to ensure $\epsilon = 10$, and (3) noise injected to ensure $\epsilon = 1$. As $\epsilon$ decreases, the performance of \texttt{VAFL} and \texttt{ZOO-VFL} deteriorates significantly. While \texttt{VFL-CZOFO} exhibits a relatively smaller accuracy drop compared to \texttt{VAFL} and \texttt{ZOO-VFL}, the reduction remains substantial, hindering the system from achieving an optimal model. In contrast, our \texttt{DPZV} maintains consistently high performance even under strict privacy budgets ($\epsilon\le1$), with only minor deviations from its No DP accuracy.  These results show the advantage of \texttt{DPZV}, where \textit{memory-efficient zeroth-order updates significantly enhance the training stability of DP-protected systems.}

\subsection{Communication Efficiency Evaluation}
\label{ssec:com_cost_setting}
We quantify the communication volume of the network by
\begin{equation}
\mathcal{V}_T \;=\; \sum_{t=1}^{T}\sum_{m=1}^{M} \bigl( N_{t,m}^{\mathrm{up}} + N_{t,m}^{\mathrm{down}} \bigr),\label{eq:VT}
\end{equation}
where \(N_{t,m}^{\mathrm{up}}\) and \(N_{t,m}^{\mathrm{down}}\) denote the sizes (in bytes) of the tensors transmitted in the uplink and downlink, respectively, between the server and device \(m\) at communication round \(t\). This metric measures the total communication cost as the cumulative number of bytes exchanged over all devices and rounds.

\textbf{Evaluation on Total Communication Cost.}
Table~\ref{tab:commu_energy} reports the total communication cost $\mathcal{V}_T$ required to reach a target accuracy, which is set to be 90\%, 40\%, 50\%, and 80\% for each dataset, respectively. The metric~\eqref{eq:VT} captures the end-to-end communication cost over the entire training process, reflecting both convergence speed and stability under differential privacy constraints.

Across the evaluated datasets, {\tt DPZV} consistently achieves low total communication cost and often outperforms both first-order and ZO baselines. In particular, {\tt DPZV} significantly reduces communication cost compared to existing ZO methods, while remaining competitive with the first-order baseline {\tt VAFL}. In contrast, ZO baselines such as {\tt VFL-CZOFO} and {\tt ZOO-VFL} incur substantially higher transmitted bytes, and in some cases fail to reach the target accuracy within the training budget. These results underscore the importance of an improved privacy--accuracy tradeoff, as poor convergence under privacy constraints directly translates into excessive communication overhead. By converging in fewer rounds under privacy constraints, {\tt DPZV} reduces the total communication cost required to attain a desired accuracy.

\begin{figure}[t!]
\centering
        \includegraphics[width=\linewidth]{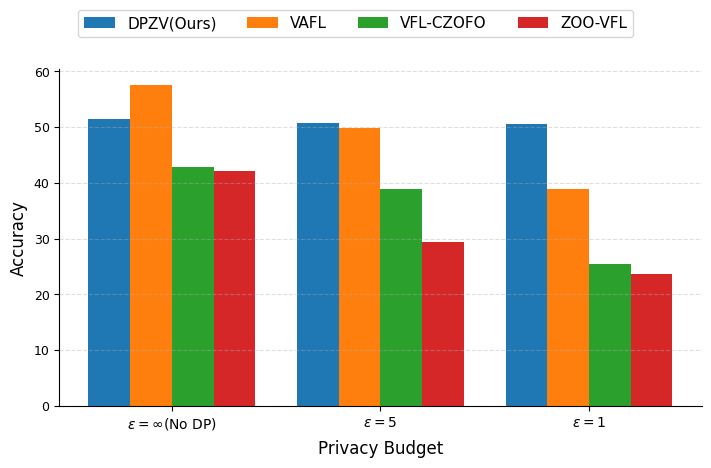}
        \vspace{-1mm}
    \caption{Achieved accuracy under fixed communication cost on CIFAR-10 under different privacy budget. $\delta$ is set to $1\times 10^{-3}$.
    % We attribute this edge to the computation efficiency of ZO optimization, which significantly reduces model delay.
    }
    \label{fig:privacy_ablation}
\end{figure}
\textbf{Communication Efficiency under Differential Privacy Budget.}
Figure~\ref{fig:privacy_ablation} evaluates model accuracy under a fixed total communication budget, highlighting how differential privacy alters the effectiveness of each transmitted byte. In the non-private setting ($\epsilon=\infty$), the first-order baseline {\tt VAFL} achieves the highest accuracy, which is consistent with standard optimization theory: in the absence of privacy constraints, first-order methods benefit from exact gradient information and therefore typically outperform ZO alternatives. In contrast, ZO methods, including {\tt DPZV}, rely on gradient estimation from function evaluations and may incur a modest performance gap when no privacy noise is present.

However, this advantage of {\tt VAFL} diminishes rapidly as privacy constraints are introduced. As the privacy budget decreases to $\epsilon=5$ and further to $\epsilon=1$, {\tt VAFL} experiences a substantial accuracy drop, despite transmitting the same number of bytes. This is because enforcing differential privacy requires injecting high-dimensional noise into gradient updates, causing a large fraction of communicated information to be dominated by noise rather than optimization signal. The degradation is even more pronounced for ZO baselines such as {\tt ZOO-VFL}. While {\tt VFL-CZOFO} exhibits improved robustness relative to {\tt ZOO-VFL}, it still incurs significant performance loss under tight privacy budgets.
In contrast, {\tt DPZV} maintains consistently strong performance across all privacy regimes, with only marginal degradation as the privacy budget tightens. By injecting scalar noise rather than high-dimensional vector noise, {\tt DPZV} effectively limits variance amplification. As a result, {\tt DPZV} achieves a markedly superior privacy–communication–accuracy tradeoff, enabling more effective utilization of the communication budget and requiring fewer rounds to reach a target accuracy under strict privacy constraints. From a communication perspective, DPZV preserves a higher signal-to-noise ratio per transmission, enabling faster accuracy recovery under the same communication budget.
\begin{figure}[t]
%\hspace{-5pt}
\begin{minipage}[t]{0.49\textwidth}
    % First main figure
\centering
\includegraphics[width=\linewidth]{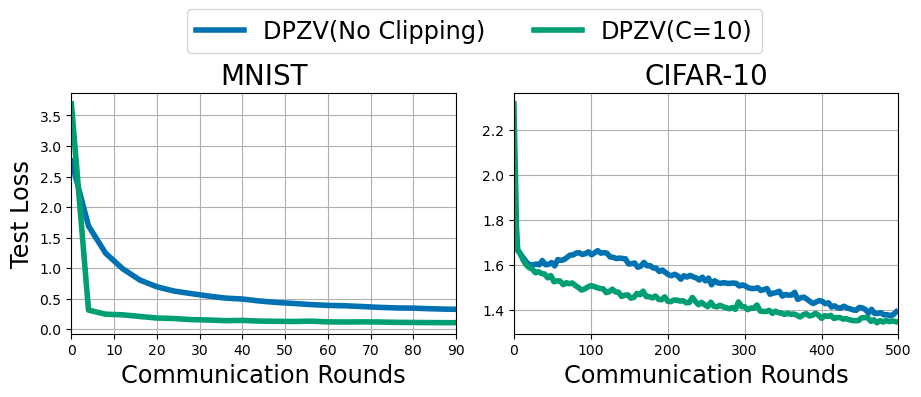}    
        \caption{Effect of gradient clipping on {\tt DPZV} under non-private settings. We compare models trained with and without clipping of the ZO information. Across both MNIST and CIFAR-10, clipping accelerates convergence and stabilizes training.}
        \label{fig:clip}
        \end{minipage}
    \hspace{10pt}
    % Second main figure
    \begin{minipage}[t]{0.49\textwidth}
        \centering

            \includegraphics[width=\linewidth]{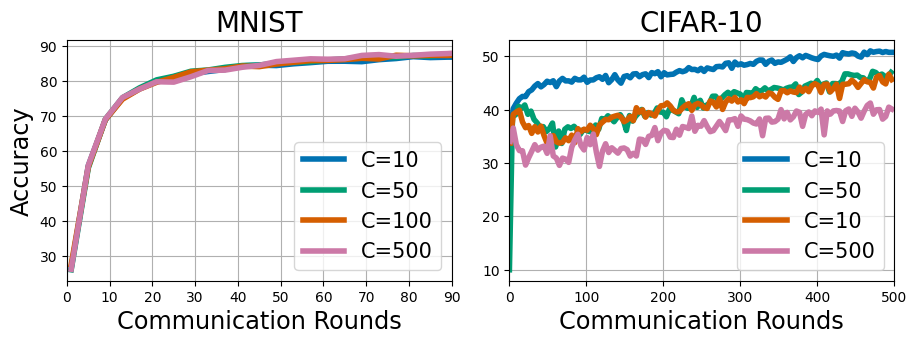}
        \caption{Effect of clipping threshold on {\tt DPZV} under differential privacy constraint $\epsilon=1$. While all clipping levels perform similarly on MNIST, smaller thresholds ($C=10$) significantly improve accuracy and convergence on CIFAR-10. }
        \label{fig:clip_levels}
    \end{minipage}
   % \vspace{-15pt}
\end{figure}
\subsection{Ablation Study}
\textbf{Clipping benefits convergence.}
We observe in Figure~\ref{fig:PA_TO} that even under no privacy constraints ($\epsilon=\infty$), {\tt DPZV} outperforms other ZO based algorithms. A key distinction lies in {\tt DPZV}'s use of scalar clipping on ZO information, while originally introduced for differential privacy, also acts as a form of gradient regularization. This regularization effect has been shown to improve convergence in prior work~\cite{zhang2019gradient}, and we observe similar benefits here.
Figure~\ref{fig:clip} verifies our insight by comparing the convergence behavior of {\tt DPZV} with and without gradient clipping under a non-private setting. The plots show test loss versus communication rounds. In both datasets, applying clipping to the ZO information significantly improves convergence speed. For MNIST, clipped {\tt DPZV} reaches low test loss much faster and stabilizes more smoothly. On CIFAR-10, the clipped version also demonstrates consistently lower loss throughout training. These results suggest that clipping not only stabilizes training but also enhances convergence efficiency, even when privacy is not enforced.

\textbf{Impact of Sensitivity Level.}
Figure~\ref{fig:clip_levels} presents the performance of {\tt DPZV} under different sensitivity levels, which is controlled by the clipping threshold $C$. We apply a fixed privacy budget of $\epsilon=1, \delta=1\times 10^{-3}$ on both MNIST and CIFAR-10. In the MNIST setting, all clipping levels achieve similar convergence and final accuracy, suggesting the model is robust to the choice of $C$ in simple tasks. In contrast, the CIFAR-10 results highlight a pronounced impact: smaller clipping values (e.g., $C=10$) yield better accuracy and stability over training. Larger thresholds, such as $C=500$ degrade performance, likely due to the excessive noise required to satisfy DP constraints. These results emphasize the importance of careful clipping calibration in more complex settings to ensure a good privacy-utility tradeoff.

% \textbf{Efficiency Gain of Embedding }

\section{Conclusion and Future Work}
% In this paper, we proposed DPZV, a differentially private zeroth-order vertical federated learning framework for privacy-critical and bandwidth-constrained settings. By eliminating explicit gradient transmission and injecting calibrated scalar noise at the server, {\tt DPZV} achieves tunable differential privacy while avoiding the variance amplification caused by high-dimensional noise in existing VFL methods. This design directly mitigates dominant privacy risks, particularly label inference from backward information, without degrading forward feature communication.

% We established rigorous convergence guarantees for {\tt DPZV} under asynchronous updates and bounded delays, and proved that the method satisfies formal $(\epsilon,\delta)$-differential privacy. These results highlight a previously underexplored advantage of zeroth-order optimization in private federated learning.
% Experimental results on image and language tasks demonstrate that {\tt DPZV} consistently outperforms both first-order and zeroth-order VFL baselines across a wide range of privacy budgets, achieving superior accuracy, stability, and memory efficiency. The results further show that scalar noise injection enables robust and scalable training under strict privacy constraints.

In this paper, we proposed {\tt DPZV}, a differentially private zeroth-order vertical federated learning framework designed for privacy-critical and bandwidth-constrained settings. By eliminating explicit gradient transmission and injecting calibrated scalar noise at the server, {\tt DPZV} achieves tunable $(\epsilon,\delta)$-differential privacy while avoiding the variance amplification associated with high-dimensional noise in existing VFL methods, thereby mitigating key privacy risks such as label inference. We established convergence guarantees under asynchronous updates and bounded delays, highlighting an underexplored advantage of zeroth-order optimization for private federated learning. Extensive experiments on image and language tasks demonstrate that {\tt DPZV} consistently outperforms both first-order and zeroth-order VFL baselines across a wide range of privacy budgets, achieving improved accuracy, stability, and scalability under strict privacy constraints.

Future work includes extending {\tt DPZV} to stronger adversarial models and adaptive privacy mechanisms, broadening the applicability of zeroth-order methods in privacy-preserving collaborative learning.

% \section*{Acknowledgments}
% This should be a simple paragraph before the References to thank those individuals and institutions who have supported your work on this article.


% \section{References Section}
% \printbibliography[heading=none]
\bibliography{references}
\bibliographystyle{IEEEtran}

 

% \end{IEEEbiography}



\vfill

\newpage
\onecolumn

% \begin{center}
%     {\bf\Large Appendix}
% \end{center}
% \setcounter{section}{0}   % <<< RESET HERE
% \renewcommand{\thesection}{\Alph{section}}

% \startcontents[sections]
% \printcontents[sections]{l}{1}{\setcounter{tocdepth}{3}}

% \newpage

\appendix
\subsection{Preliminaries and Outline}
\label{appen:prelim}
We first define the following notation table to facilitate the proof: 
\begin{table}[ht]
    \centering
    % Adjust the column specifications (e.g., l, c, r, p{width}) as needed
    \begin{tabular}{p{0.4\textwidth} p{0.5\textwidth}}
    \toprule
    \textbf{Notation} & \textbf{Description} \\
    \midrule
    $\vw=[w_0,w_1,w_2,\dots,w_M]$ & All learnable parameters\\
    $d_0, d_1,\dots,d_M$ & Dimension of parameters on server (machine $0$) and device $1,\dots,M$\\
    $d=\sum_{m=0}^{M}d_m$ & Dimension of all parameters\\
    $f(\vw;\xi_i):=\gL(w_0, h_{1,i}, h_{2,i},\dots, h_{M,i};y_i)$ & Loss function with regard to datum with ID $i$\\
    $F(\vw):=F(\vw;\gD,\mY)$ & Global loss function\\
    $\chi^t = [w_1^{t_1}, \dots, w_M^{t_M}]$ & Latest learnable parameters of all devices at server time $t$\\
    $\Tilde{\chi}^t = [w_1^{t_1-\tau_1^t}, \dots, w_M^{t_M-\tau_M^t}]$ & Delayed learnable parameters of all devices at server time $t$\\
    $\vw^t =[w_0^{t}, w_1^{t_1}, \dots, w_M^{t_M}]$ & Latest learnable parameters of all devices and the server at server time $t$ \\
    $\Tilde{\vw}^t =[w_0^{t-\tau_1^t}, w_1^{t_1-\tau_1^t}, \dots, w_M^{t_M-\tau_M^t}]$ & Delayed learnable parameters of all devices and the server at server time $t$ \\
    $h_{m,i}^{t_m\pm}=h_{m}(w_m^{t_m}\pm\lambda_m \vu_{m}^{t_m};\xi_{m,i})$ & Local embeddings of device $m$ for data sample $i$ at device time $t_m$ under the perturbed parameters \\
    $\delta_{m,i}^{t,t_m}$ as defined in \eqref{eq:delta} & zeroth-order difference information from device $m$ and data sample $i$ at server time $t$ and device time $t_m$\\
    $g^t_{m,i}(\Tilde{\vw}^t) = \delta_{m,i}^{t,t_m} \vu_{m}^{t_m}$ & zeroth-order gradient estimator from device $m$ and data sample $i$ at server time $t$ (with delay)\\
    % $g^t_{m,i}(\vw^t)$ & Zeroth-order gradient estimator from client $m$ and data sample $i$ at server time $t$ (without delay)\\
    $\breve{g}^t_{m,i}(\Tilde{\vw}^t) = \mathrm{clip}_C\left(\delta_{m,i}^{t,t_m}\right)\vu_m^{t_m}$ & Clipped zeroth-order gradient estimator from device $m$ and data sample $i$ at server time $t$ (with delay)\\
    $\breve{G}^t_{m}(\Tilde{\vw}^t) = (1/B)\sum_{i\in \gI_{m}^{t_m}}\breve{g}^t_{m,i}(\Tilde{\vw}^t)+z_{m}^{t_m}\vu_m^{t_m}$ & Clipped differential private zeroth-order gradient estimator from device $m$ at server time $t$ (with delay)\\ 
    $G^t_{m}(\Tilde{\vw}^t) =(1/B) \sum_{i\in \gI_{m}^{t_m}}g^t_{m,i}(\Tilde{\vw}^t)+z_{m}^{t_m}\vu_m^{t_m}$ & Non-clipped differential private zeroth-order gradient estimator from device $m$ at server time $t$ (with delay)\\ 
    \bottomrule
    \end{tabular}
    \vspace{5pt}
    \caption{Table of Notations}
    \label{tab:notation}
\end{table}

Note that in the notation table, we use ``$\breve{\phantom{x}}$" to define clipped gradient estimators, and we use ``$\Tilde{\phantom{x}}$" to denote delayed model parameters. In the rest of the proof, we also use gradient estimators parameterized by the no delaying parameters $\vw$ instead of $\Tilde{\vw}$ to assume that we update the model without delay.
To begin with, we restate the assumptions required for establishing the convergence analysis.
\begin{assumption}[$\ell$-Lipschitz]\label{assum:Lip_full}
    The function $f(\vw;\xi)$ is $\ell$-Lipschitz continuous for every $\xi$. 
\end{assumption}
\begin{assumption}[$L$-Smooth]\label{assum:smooth_full}
    The function $f(\vw;\xi)$ is $L$-Smooth for every $\xi$. Specifically, there exists an $L>0$ for all $m=0, \dots, M$ such that $\norm{\nabla_{w_m}f(\vw)-\nabla_{w_m}f(\vw')}\le L \norm{\vw-\vw'}$.
\end{assumption}
\begin{assumption}[Bounded gradient variance]\label{assum:bound_full}
    The variance of the stochastic first order gradient is upper bounded in expectation:
    \[\E\left[\norm{\nabla_{\vw}f(\vw;\xi)-\nabla_{\vw}F(\vw)}^2\right]\le \sigma_s^2\]
\end{assumption}
\begin{assumption}[Independent Participation]\label{assum:ind_part_full}
    The probability of one device participating in one communication round is independent of other devices and satisfies
    \[\Prob(\text{device }m\text{ uploading}) = q_m \]
    Specially, we set $q_0=1$ as the server always participates in the update.
\end{assumption}
One of the important parts in the proof of Theorem \ref{thm:main} is to bound the zeroth-order gradient estimator. We first introduce the formal definition of zeroth-order two-point gradient estimator that is used in our algorithm, and prove some technical lemmas that reveal some important properties.
\begin{definition}
    Let $\vu$ be uniformly sampled from the Euclidean sphere $\sqrt{d} \mathbb{S}^{d-1}$. For any function $f(x) : \mathbb{R}^d \to \mathbb{R}$ and $\lambda > 0$, we define its zeroth-order gradient estimator as 
\begin{align}\label{eq:zero_def}
    g(\vw) = \frac{(f(\vw + \lambda \vu) - f(\vw - \lambda \vu))}{2\lambda} \vu
\end{align}
\end{definition}
\begin{lemma}[Restatement of Lemma \ref{lem:zero}]
Let $g(\vw)$ be the zeroth-order gradient estimator defined as in \eqref{eq:zero_def}, with $f(\vw)$ being the loss function. We define the smoothed function $f_\lambda(\vw) = \E_{\vu}[f(\vw + \lambda \vu)]$, where $v$ is uniformly sampled from the Euclidean ball $\sqrt{d} \mathbb{B}^d = \{\vw \in \mathbb{R}^d \mid \norm{\vw} \leq \sqrt{d}\}$. The following properties hold:
\begin{enumerate}[label=$(\roman*)$]
    \item $f_\lambda(\vw)$ is differentiable and $\E_\vu[g(\vw)] = \nabla f_\lambda(\vw)$.
    \item If $f(\vw)$ is $L$-smooth, we have that
    \begin{align}
        \norm{\nabla f(\vw) - \nabla f_\lambda(\vw)} \le \frac{L}{2} \lambda d^{3/2},
    \end{align}
    \begin{align}
        | f(\vw) -  f_\lambda(\vw)| \le \frac{L}{2} \lambda^2 d,
    \end{align}
    and
    \begin{align}
    \E_\vu[\norm{g_\lambda(\vw)}^2] \le 2d \cdot \norm{\nabla f(\vw)}^2 + \frac{L^2}{2} \lambda^2 d^3.
    \end{align}
\end{enumerate}
\end{lemma}
% \begin{remark}
Based on \eqref{eq:func_diff}, we can further show:
    \begin{align}
        \norm{\nabla f_\lambda(\vw)}^2\le 2\norm{\nabla f(\vw)}^2 +\frac{L^2}{2} \lambda^2 d^3\label{eq:func_squared_1}
    \end{align}
    \begin{align}
        \norm{\nabla f(\vw)}^2\le 2\norm{\nabla f_\lambda(\vw)}^2 +\frac{L^2}{2} \lambda^2 d^3\label{eq:func_squared_2}
    \end{align}
% \end{remark}
This is the standard result of zeroth-order optimization. The proof of the Lemma is given by \cite{nesterov2017random}
We also find the following lemmas useful in the proof:
\begin{lemma}\label{lem:u}
Let $\vu$ be uniformly sampled from the Euclidean sphere $\sqrt{d}S^{d-1}$, and $\va$ be any vector of constant value.
We have that $\E[\vu] = 0$ and
    \[
        \Prob(|\vu^\top \va| \ge C) \le 2\sqrt{2\pi} \exp\left(-\frac{C^2}{8\norm{\va}^2}\right).
    \]
\end{lemma}
\begin{proof}
    This lemma follows exactly from Lemma C.1. in \cite{zhang2024dpzero}.
\end{proof}

\begin{lemma}[Restatement of Lemma \ref{lem:Q}]
    Let $Q$ be the event that clipping happened for a sample $\xi$, $d$ be the model dimension, and $L,\ell$ be the Lipschitz and smooth constant as defined in assumption \ref{assum:Lip_full} and assumption \ref{assum:smooth_full}. For $\forall C_0>0$, if the clipping threshold $C$ follows $C\ge C_0+L\lambda d/2$, we have the following upper bound for the probability of clipping:
    \begin{align}
        P = \Prob(Q)\le 2\sqrt{2\pi}\exp(-\frac{C_0^2}{8\ell^2})=\Xi
    \end{align}
\end{lemma}
\begin{proof}
Since $f(\vu;\xi)$ is $L$-Smooth for every $\xi$, we have
\begin{align*}
\frac{|f(\vw + \lambda  \vu; \xi) - f(\vw - \lambda  \vu; \xi)|}{2\lambda}
&\le
|u^\top \nabla f(\vu; \xi)|
+
\frac{|f(\vw + \lambda  \vu; \xi) - f(\vw; \xi) 
      - \lambda  \vu^\top \nabla f(\vw; \xi)|}{2\lambda}\\
&+
\frac{|f(\vw - \lambda  \vu; \xi) - f(\vw; \xi)
      + \lambda  \vu^\top \nabla f(\vw; \xi)|}{2\lambda}\\
&\le
|\vu^\top \nabla f(\vw; \xi)| +\frac{L}{2}\lambda.
\end{align*}
Therefore, by Lemma \ref{lem:u} and Assumption \ref{assum:Lip_full}, we obtain
\begin{align*}
\Prob(Q)&=\Prob(
  \frac{| f(\vw + \lambda \vu; \xi_i) - f(\vw - \lambda  \vu; \xi_i)|}
       {2\lambda}
  \ge C_0 + \frac{L}{2}\lambda 
)\\
\le&
\Prob(|\vu^\top \nabla f(\vw; \xi_i)|\ge C_0)\\
\le&
2\sqrt{2\pi}\exp(-\frac{C_0^{2}}{8\|\nabla f(\vw; \xi_i)\|^{2}})\\
\le&
2\sqrt{2\pi}\exp(-\frac{C_0^{2}}{8\ell^{2}}).
\end{align*}
\end{proof}

\begin{lemma}[Expectation and Variance of Clipped Zeroth-order Gradient Estimator]Recall that  $\breve{g}^t_{m,i}(\vw^t)$ is defined as the clipped zeroth-order gradient estimator assuming no communication delay, random perturbation $\vu$ is defined in Lemma \ref{lem:u}, and event $Q$ is defined in  Lemma \ref{lem:Q}. We have the following properties:
\begin{enumerate}[label=$(\roman*)$]
\item When taking expectation w.r.t $\vu$ and $Q$, the clipped zeroth-order gradient estimator follows \begin{align}
    \E_\vu[\breve{g}^t_{m,i}(\vw^t)] = (1-P)\nabla_{w_m} F_\lambda(\vw^t)\label{eq:mean}
\end{align}
\item The variance of $\breve{g}^t_{m,i}(\vw^t)$ follows
\begin{align*}
    \Var(\breve{g}^t_{m,i}(\vw^t))\le& (1-P)(2d_m \norm{\nabla_{w_m} F(\vw^t)}^2 + 2d_m\sigma_s^2 + \frac{L^2}{2} \lambda^2 d_m^3 )
    +P C^2 d_m-(1-P)^2\norm{\nabla_{w_m} F_\lambda(\vw^t)}^2\numbereq\label{eq:var}
\end{align*}
\end{enumerate}
\end{lemma}
\begin{proof}
    For $(i)$, we have
\begin{align*}
    \E\left[\breve{g}^t_{m,i}(\vw^t)\right] =&\E\left[\breve{g}^t_{m,i}(\vw^t)|\bar{Q}\right]\Prob(\bar{Q})+ \E\left[\breve{g}^t_{m,i}(\vw^t)|Q\right]\Prob(Q)\\
    =&\E\left[\Tilde{g}^t_{m,i}(\vw^t)\right](1-\Prob(Q))+ \E\left[C\vu_m^t\right]\Prob(Q)\\
    =&(1-P)\nabla_{w_m} F_\lambda(\vw^t)
\end{align*}
where in the first step we applied the Law of Total Expectation, and in the last step we used the property $(i)$ in Lemma \ref{lem:zero} and $(i)$ in Lemma \ref{lem:u}.\\
By \eqref{eq:mean}, we can further bound the variance of $\breve{g}^t_{m,i}$
\begin{align*}
    &\Var(\breve{g}^t_{m,i}(\vw^t))\\
    =&\E\left[\norm{\breve{g}^t_{m,i}(\vw^t)-(1-P)\nabla_{w_m} F_\lambda(\vw^t)}^2\right] \\
    =&\E\left[\norm{\breve{g}^t_{m,i}(\vw^t)}^2\right]-(1-P)^2\norm{\nabla_{w_m} F_\lambda(\vw^t)}^2 \\
    =&\E\left[\norm{\Tilde{g}^t_{m,i}(\vw^t)}^2|\bar{Q}\right]\Prob(\bar{Q})+\E\left[C^2 \norm{\vu_m^t}^2|Q\right]\Prob(Q)-(1-P)^2\norm{\nabla_{w_m} F_\lambda(\vw^t)}^2 \\
    \overset{1)}{\le} &(1-P)(2d_m \norm{\nabla_{w_m} f(\vw^t;\xi_{m,t})}^2 + \frac{L^2}{2} \lambda^2 d_m^3 )\\
    +&P C^2 d_m-(1-P)^2\norm{\nabla_{w_m} F_\lambda(\vw^t)}^2 \\
    \overset{2)}{\le} &(1-P)(2d_m \norm{\nabla_{w_m} F(\vw^t)}^2 + 2d_m \sigma_s^2 + \frac{L^2}{2} \lambda^2 d_m^3 )\\
    +&P C^2 d_m-(1-P)^2\norm{\nabla_{w_m} F_\lambda(\vw^t)}^2 \\
\end{align*}
where $1)$ is by the property of zeroth-order gradient estimator \eqref{eq:zero_grad} and $2)$ follows from the bounded gradient assumption(Assumption \ref{assum:bound_full}).
\end{proof}

\begin{lemma}[Restatement of Lemma \ref{lem:var}]Let $\breve{G}^t_{m}(\vw^t)$ be the Differential Private Zeroth-order Gradient without delay. Under the same condition as Lemma \ref{lem:zero}, we can bound the variance of $\breve{G}^t_{m}(\vw^t)$ in expectation, with the expectation taken on random direction $\vu$, DP noise $z$, and clipping event $Q$: 
\begin{align*}
\Var(\breve{G}^t_{m}(\vw^t))\le&\frac{1-P}{B}(2d_m \E\left[\norm{\nabla_{w_m} F(\vw^t)}^2\right] + 2d_m\sigma_s^2 + \frac{L^2}{2} \lambda^2 d_m^3 )+\frac{P}{B} C^2 d_m\\
&-\frac{(1-P)^2}{B}\E\left[\norm{\nabla_{w_m} F_\lambda(\vw^t)}^2\right]+\sigma_{dp}^2 d_m
\numbereq
\end{align*}
\end{lemma}
\begin{proof}
    First we show that the expectation on $\breve{G}^t_{m}(\vw^t)$ can be written as:
    $$\E_{u}[\breve{G}^t_{m}(\vw^t)]=\frac{1}{B}\sum_{i\in \gI_{m}^{t_m}} \E\left[\breve{g}^t_{m,i}(\vw^t)\right]+\E\left[z^t_{m}\vu^t_m\right] = (1-P)\nabla_{w_m} F_\lambda(\vw^t)$$
\\
Thus, the variance can be bounded by:
\begin{align*}
    &\Var(\breve{G}^t_{m}(\vw^t))\\
    =&\E\left[\norm{\breve{G}^t_{m}(\vw^t)-(1-P)\nabla_{w_m} F_\lambda(\vw^t)}^2\right]\\
    =&\E\left[\norm{\frac{1}{B}\sum_{i\in \gI_{m}^{t_m}} \left(\breve{g}^t_{m,i}(\vw^t)-(1-P)\nabla_{w_m} F_{\lambda}^t(\vw^t)\right)+z^t_{m}\vu^t_m}^2\right]\\
    =&\frac{1}{B^2}\sum_{i\in \gI_{m}^{t_m}}\E\left[\norm{\breve{g}^t_{m,i}(\vw^t)-(1-P)\nabla_{w_m} F_{\lambda}^t(\vw^t)}^2\right]+\E\left[\norm{z^t_{m}\vu^t_m}^2\right]\\
    \overset{(a)}{\leq} &\frac{1-P}{B}(2d_m \E\left[\norm{\nabla_{w_m} F(\vw^t)}^2\right] + 2d_m\sigma_s^2 + \frac{L^2}{2} \lambda^2 d_m^3 )+\frac{P}{B} C^2 d_m-\frac{(1-P)^2}{B}\E\left[\norm{\nabla_{w_m} F_\lambda(\vw^t)}^2\right]+\sigma_{dp}^2 d_m,
\end{align*}
where $(a)$ follows from \eqref{eq:var}.
\end{proof}
    
\subsection{Proof of Intermediate Lemmas}
\label{appen:lemma}
\begin{lemma}[Restatement of Lemma \ref{lem:main_1}]
     Under Assumption \ref{assum:Lip_full} to \ref{assum:ind_part_full}, we have the following lemma:
     \begin{align*}
         \E\left[F(\vw^{t+1})-F(\vw^{t}) \right] \le& -\sum_{m=0}^M q_m\eta_m(1-P)\left(\frac{1}{4}-\frac{\eta_m L (B+8d_m)}{2B}\right)\E\left[\norm{\nabla_{w_{m}}F(\vw^t)}^2\right]\\
         &+\sum_{m=0}^{M}q_m\eta_m L \left(\frac{1}{2}+2\eta_m L^2\right)\E\left[\norm{\Tilde{\chi}^t-\chi^t}^2\right]+\gA_1\numbereq
     \end{align*}
\end{lemma}
\begin{proof}
    By assumption \ref{assum:smooth_full}:
    \begin{align*}
F_\lambda(\vw^{t+1})\le&F_\lambda(\vw^{t})+\langle \nabla_{\vw}F_\lambda(\vw^t), \vw^{t+1}-\vw^t \rangle+\frac{L}{2}\norm{\vw^{t+1}-\vw^t}^2\\
=& F_\lambda(\vw^{t})-\eta_0 \langle \nabla_{w_0}F_\lambda(\vw^t), \breve{G}^t_{0}(\Tilde{\vw}^t)\rangle + \frac{L\eta_0^2}{2}\norm{\breve{G}^t_{0}(\Tilde{\vw}^t)}^2\\
&-\eta_m \langle \nabla_{w_m}F_\lambda(\vw^t), \breve{G}^t_{m}(\Tilde{\vw}^t)\rangle + \frac{L\eta_m^2}{2}\norm{\breve{G}^t_{m}(\Tilde{\vw}^t)}^2\\
\E\left[F_\lambda(\vw^{t+1})\right]
    \le& \E\left[F_\lambda(\vw^{t})\right]\underbrace{-\eta_0 \E\langle \nabla_{w_0}F_\lambda(\vw^t), \breve{G}^t_{0}(\Tilde{\vw}^t)\rangle}_{\gE_1} + \underbrace{\frac{L\eta_0^2}{2}\E\left[\norm{\breve{G}^t_{0}(\Tilde{\vw}^t)}^2\right]}_{\gE_2}\\
    &\underbrace{-\eta_{m_k} \E\langle \nabla_{w_{m_k}}F_\lambda(\vw^t), \breve{G}^t_{m_k}(\Tilde{\vw}^t)\rangle}_{\gE_3}
    + \underbrace{\frac{L\eta_{m_k}^2}{2}\E\left[\norm{\breve{G}^t_{m_k}(\Tilde{\vw}^t)}^2\right]}_{\gE_4}\\\numbereq\label{eq:first_step}
\end{align*}
where in the second step we take expectation on both sides, first w.r.t. the random direction $u$, DP noise $z$, and the clipping event $Q$, then w.r.t the device $m_k$.
We bound $\gE_1$ as the following:
\begin{align*}
    &-\eta_0\E\langle \nabla_{w_0}F_\lambda(\vw^t), \breve{G}^t_{0}(\Tilde{\vw}^t)\rangle \\
    =& -\eta_0\E\langle\nabla_{w_0}F_\lambda(\vw^t), \breve{G}^t_{0}(\Tilde{\vw}^t)-(1-P)\nabla_{w_0} F_\lambda(\Tilde{\vw}^t)+(1-P)\nabla_{w_0} F_\lambda(\Tilde{\vw}^t)\rangle\\
    = & -\eta_0\E\langle \nabla_{w_0}F_\lambda(\vw^t), \breve{G}^t_{0}(\Tilde{\vw}^t)-(1-P)\nabla_{w_0} F_\lambda(\Tilde{\vw}^t)\rangle\\
    -&\eta_0\E\langle \nabla_{w_0}F_\lambda(\vw^t), (1-P)\nabla_{w_0} F_\lambda(\Tilde{\vw}^t)-(1-P)\nabla_{w_0} F_\lambda(\vw^t)+(1-P)\nabla_{w_0} F_\lambda(\vw^t)\rangle\\
    \overset{1)}{=}& -(1-P)\eta_0\E\langle \nabla_{w_0}F_\lambda(\vw^t), \nabla_{w_0} F_\lambda(\Tilde{\vw}^t)-\nabla_{w_0} F_\lambda(\vw^t)\rangle\\
    -&\eta_0\E\langle \nabla_{w_0}F_\lambda(\vw^t),(1-P)\nabla_{w_0} F_\lambda(\vw^t)\rangle\\
    \overset{2)}{\le} & \frac{(1-P)\eta_0}{2}\E\left[\norm{\nabla_{w_0}F_\lambda(\vw^t)}^2\right] +  \frac{(1-P)\eta_0}{2} \E\left[\norm{\nabla_{w_0} F_\lambda(\Tilde{\vw}^t)-\nabla_{w_0} F_\lambda(\vw^t)}^2\right]-(1-P)\eta_0\E\left[\norm{\nabla_{w_0}F_\lambda(\vw^t)}^2\right] \\
    \overset{3)}{\le} & -\frac{(1-P)\eta_0}{2}\E\left[\norm{\nabla_{w_0}F_\lambda(\vw^t)}^2\right] +  \frac{\eta_0 L}{2} \E\left[\norm{\Tilde{\chi}^t-\chi^t}^2\right]\numbereq\label{eq:E1}
\end{align*}
where in $1)$ we use the fact that $\E\left[\breve{G}^t_{0}(\Tilde{\vw}^t)-(1-P)\nabla_{x_0} F_\lambda(\Tilde{\vw}^t)\right]=0$, in $2)$ we applied the Cauchy–Schwarz inequality,
% i.e., \[\langle A,B\rangle\le\frac{1}{2}\norm{A}^2+ \frac{1}{2}\norm{B}^2.\]
and $3)$ follows by the smoothness of $F_\lambda$ and the fact that $1-P \le 1$

For $\gE_2$, we can further bound it based on Assumption~\ref{assum:smooth_full}:
\begin{align*}
    &\frac{1}{2}\E\left[\norm{\breve{G}^t_{0}(\Tilde{\vw}^t)}^2\right]\\
    =& \frac{1}{2}\E\left[\norm{\breve{G}^t_{0}(\Tilde{\vw}^t)-(1-P)\nabla_{x_0} F_\lambda(\vw^t)+(1-P)\nabla_{x_0} F_\lambda(\vw^t)}^2\right]\\
    \overset{1)}{\le}&\E\left[\norm{\breve{G}^t_{0}(\Tilde{\vw}^t)-(1-P)\nabla_{x_0} F_\lambda(\vw^t)}^2\right]+(1-P)^2\E\left[\norm{\nabla_{x_0} F_\lambda(\vw^t)}^2\right]\\
    =&\E\left[\norm{\breve{G}^t_{0}(\Tilde{\vw}^t)-(1-P)\nabla_{x_0} F_\lambda(\Tilde{\vw}^t)+(1-P)\nabla_{x_0} F_\lambda(\Tilde{\vw}^t)-(1-P)\nabla_{x_0} F_\lambda(\vw^t)}^2\right]\\
    &+(1-P)^2\E\left[\norm{\nabla_{x_0} F_\lambda(\vw^t)}^2\right]\\
    \overset{2)}{\le}&2\E\left[\norm{\breve{G}^t_{0}(\Tilde{\vw}^t)-(1-P)\nabla_{x_0} F_\lambda(\Tilde{\vw}^t)}^2\right]+2(1-P)^2\E\left[\norm{\nabla_{x_0} F_\lambda(\Tilde{\vw}^t)-\nabla_{x_0} F_\lambda(\vw^t)}^2\right]\\
    &+(1-P)^2\E\left[\norm{\nabla_{x_0} F_\lambda(\vw^t)}^2\right]\\
    \overset{3)}{\le}& \frac{2(1-P)}{B}(2d_0 \E\left[\norm{\nabla_{x_0} F(\vw^t)}^2\right] + 2d_0\sigma_s^2 + \frac{L^2}{2} \lambda^2 d_0^3 )+\frac{2P}{B} C^2 d_0\\
    &-\frac{2(1-P)^2}{B}\E\left[\norm{\nabla_{x_0} F_\lambda(\vw^t)}^2\right]+2\sigma_{dp}^2 d_0 +2(1-P)^2L^2\E\left[\norm{\Tilde{\chi}^t-\chi^t}^2\right]+(1-P)^2\E\left[\norm{\nabla_{x_0} F_\lambda(\vw^t)}^2\right]\\
    \overset{4)}{\le}&\frac{4(1-P)d_0}{B}\E\left[\norm{\nabla_{x_0} F(\vw^t)}^2\right]+(1-P)\E\left[\norm{\nabla_{x_0} F_\lambda(\vw^t)}^2\right]+2 L^2\E\left[\norm{\Tilde{\chi}^t-\chi^t}^2\right]+\gG_{0}\numbereq\label{eq:E2}
\end{align*}
where in $1)$ and $2)$ we applied the Cauchy–Schwarz inequality,
% \[\norm{A+B}^2\le2\norm{A}^2+ 2\norm{B}^2\]
and in $3)$ we substitute \eqref{eq:G_var} in and use the $L$-smoothness of $F_\lambda$, and in $(iv)$ we use the fact that $1-P \le 1$ and let 
\[\gG_{0} = \frac{4}{B}d_0\sigma_s^2 + \frac{L^2}{B} \lambda^2 d_0^3 +\frac{2P}{B} C^2 d_0+2\sigma_{dp}^2 d_0\]

Similarly, For $\gE_3$:
\begin{align}
    -\eta_{m_k}\E\langle \nabla_{w_{m_k}}F_\lambda(\vw^t), \breve{G}^t_{m}(\Tilde{\vw}^t)\rangle\le -\frac{(1-P)\eta_{m_k}}{2}\E\left[\norm{\nabla_{w_{m_k}}F_\lambda(\vw^t)}^2\right] +  \frac{\eta_{m_k} L}{2} \E\left[\norm{\Tilde{\chi}^t-\chi^t}^2\right]\label{eq:E3}
\end{align}

And for $\gE_4$:
\begin{align}
    \frac{1}{2}\E\left[\norm{\breve{G}^t_{m_k}(\Tilde{\vw}^t)}^2\right]
    \le&\frac{4(1-P)d_{m}}{B}\E\left[\norm{\nabla_{w_{m_k}} F(\vw^t)}^2\right]+(1-P)\E\left[\norm{\nabla_{w_{m_k}} F_\lambda(\vw^t)}^2\right]\\
    &+2L^2\E\left[\norm{\Tilde{\chi}^t-\chi^t}^2\right]+\gG_{m}\label{eq:E4}
\end{align}
where we let 
\begin{align}
\gG_{m} = \frac{4}{B}d_m\sigma_s^2 + \frac{L^2}{B} \lambda^2 d_m^3 +\frac{2P}{B} C^2 d_m+2\sigma_{dp}^2 d_m\label{eq:gm}
\end{align}
Substituting \eqref{eq:E1}, \eqref{eq:E2}, \eqref{eq:E3}, and \eqref{eq:E4} into \eqref{eq:first_step}, we have
\begin{align*}
    &\E\left[F(\vw^{t+1})-F(\vw^{t}) \right]\\
    \le&\E\left[F_\lambda(\vw^{t})\right]-\frac{(1-P)\eta_0}{2}\E\left[\norm{\nabla_{w_0}F_\lambda(\vw^t)}^2\right] +  \frac{\eta_0 L}{2} \E\left[\norm{\Tilde{\chi}^t-\chi^t}^2\right]\\
    +&\frac{4(1-P)d_0 L \eta_0^2}{B}\E\left[\norm{\nabla_{w_0} F(\vw^t)}^2\right]+(1-P) L \eta_0^2 \E\left[\norm{\nabla_{w_0} F_\lambda(\vw^t)}^2\right]+2 L^3\eta_0^2\E\left[\norm{\Tilde{\chi}^t-\chi^t}^2\right]+L\eta_0^2\gG_{0}\\
    -&\frac{(1-P)\eta_{m_k}}{2}\E\left[\norm{\nabla_{w_{m_k}}F_\lambda(\vw^t)}^2\right] +  \frac{\eta_{m_k} L}{2} \E\left[\norm{\Tilde{\chi}^t-\chi^t}^2\right]\\
    +&\frac{4(1-P)d_{m_k} L\eta_{m_k}^2}{B}\E\left[\norm{\nabla_{w_{m_k}} F(\vw^t)}^2\right]+(1-P)L\eta_{m_k}^2\E\left[\norm{\nabla_{w_{m_k}} F_\lambda(\vw^t)}^2\right]\\
    +&2L^3\eta_{m_k}^2\E\left[\norm{\Tilde{\chi}^t-\chi^t}^2\right]+L \eta_{m_k}^2\gG_{m}\\
    \le&\E\left[F_\lambda(\vw^{t})\right]-\eta_0(1-P)\left(\frac{1}{2}-\eta_0 L\right)\E\left[\norm{\nabla_{w_0}F_\lambda(\vw^t)}^2\right] + \eta_0 L \left(\frac{1}{2}+2\eta_0 L^2\right)\E\left[\norm{\Tilde{\chi}^t-\chi^t}^2\right]\\
    +&\eta_0^2\frac{4(1-P)d_0 L }{B}\E\left[\norm{\nabla_{w_0} F(\vw^t)}^2\right]+\eta_0^2L\gG_{0}\\
    -&\sum_{m=1}^M q_m\eta_{m}(1-P)\left(\frac{1}{2}-\eta_{m} L\right)\E\left[\norm{\nabla_{w_m}F_\lambda(\vw^t)}^2\right] + \sum_{m=1}^M q_m\eta_{m} L \left(\frac{1}{2}+2\eta_{m} L^2\right)\E\left[\norm{\Tilde{\chi}^t-\chi^t}^2\right]\\
    +&\sum_{m=1}^{M} q_m \eta_{m}^2 \frac{4(1-P)d_m L}{B}\E\left[\norm{\nabla_{w_m} F(\vw^t)}^2\right]+\sum_{m=1}^M q_m \eta_{m}^2 L \gG_{m}\numbereq\label{eq:combine}
\end{align*}
where in the last inequality we further take expectation w.r.t. device $M$ and combine similar terms.

From \eqref{eq:combine}, we utilize the properties of the smooth function \eqref{eq:func_diff} and \eqref{eq:func_squared_2} to turn all the smooth function $F_\lambda$ into the true loss function $F$:
\begin{align*}
    &\E\left[F(\vw^{t+1})-F(\vw^{t}) \right]\overset{1)}{\le} \E\left[F_\lambda(\vw^{t+1})-F_\lambda(\vw^{t})\right] + L\lambda^2 d\\
    \overset{2)}{\le}&- \eta_0(1-P)\left(\frac{1}{2}-\eta_0 L\right)\left(\frac{1}{2}\E\left[\norm{\nabla_{x_{0}}F(\vw^t)}^2\right]-\frac{L^2}{4}\lambda^2 d_0^3\right) +  \eta_0 L \left(\frac{1}{2}+2\eta_0 L^2\right)\E\left[\norm{\Tilde{\chi}^t-\chi^t}^2\right]\\
    +&\eta_0^2\frac{4(1-P)d_0 L}{B}\E\left[\norm{\nabla_{x_{0}} F(\vw^t)}^2\right]+ \eta_0^2 L \gG_{0}\\
    -&\sum_{m=1}^M q_m\eta_m(1-P)\left(\frac{1}{2}-\eta_m L\right)\left(\frac{1}{2}\E\left[\norm{\nabla_{w_{m}}F(\vw^t)}^2\right]-\frac{L^2}{4}\lambda^2 d_m^3\right) \\
    +&  \sum_{m=1}^{M}q_m\eta_m L \left(\frac{1}{2}+2\eta_m L^2\right)\E\left[\norm{\Tilde{\chi}^t-\chi^t}^2\right]\\
    +&\sum_{m=1}^{M} q_m \eta_m^2\frac{4(1-P)d_m L}{B}\E\left[\norm{\nabla_{w_{m}} F(\vw^t)}^2\right]+\sum_{m=1}^M q_m\eta_m^2 L \gG_{m} + L\lambda^2 d\\
    \overset{3)}{\le}&-\sum_{m=0}^M q_m\eta_m(1-P)\left(\frac{1}{4}-\frac{\eta_m L (B+8d_m)}{2B}\right)\E\left[\norm{\nabla_{w_{m}}F(\vw^t)}^2\right]\\
    + & \sum_{m=0}^{M}q_m\eta_m L \left(\frac{1}{2}+2\eta_m L^2\right)\E\left[\norm{\Tilde{\chi}^t-\chi^t}^2\right]+\sum_{m=0}^M q_m\eta_m\left(\frac{1}{2}-\eta_m L\right)\frac{L^2}{4}\lambda^2 d_m^3\\
    +&\sum_{m=0}^M q_m \eta_m^2 L \gG_{m} + L\lambda^2 d\\
    \overset{4)}{\le}&-\sum_{m=0}^M q_m\eta_m(1-P)\left(\frac{1}{4}-\frac{\eta_m L (B+8d_m)}{2B}\right)\E\left[\norm{\nabla_{w_{m}}F(\vw^t)}^2\right]\\
    +&\sum_{m=0}^{M}q_m\eta_m L \left(\frac{1}{2}+2\eta_m L^2\right)\E\left[\norm{\Tilde{\chi}^t-\chi^t}^2\right]\\
    + &\sum_{m=0}^M q_m\eta_m \frac{L^2}{8}\lambda^2 d_m^3+ \sum_{m=0}^M q_m \eta_m^2 L \left(\frac{4}{B}d_m\sigma_s^2 + \frac{(4-B)L^2}{4B} \lambda^2 d_m^3 +\frac{2P}{B} C^2 d_m+2\sigma_{dp}^2 d_m\right) + L\lambda^2 d\\
    \overset{5)}{\le}&-\sum_{m=0}^M q_m\eta_m(1-P)\left(\frac{1}{4}-\frac{\eta_m L (B+8d_m)}{2B}\right)\E\left[\norm{\nabla_{w_{m}}F(\vw^t)}^2\right]\\
    + &\sum_{m=0}^{M}q_m\eta_m L \left(\frac{1}{2}+2\eta_m L^2\right)\E\left[\norm{\Tilde{\chi}^t-\chi^t}^2\right]+\gA_1\numbereq\label{eq:main_1}
\end{align*}
where $1)$ and $2)$ follows from equation \eqref{eq:func_diff} and \eqref{eq:func_squared_2} in Lemma \ref{lem:zero} respectively. In $3)$, we let $q_0=1$ and combine similar terms. In $4)$, we substitute in \eqref{eq:gm}. Lastly, in $5)$, we denote
\begin{align}\label{eq:A1}
    \gA_1 = \sum_{m=0}^M q_m\eta_m \frac{L^2}{8}\lambda^2 d_m^3+ \sum_{m=0}^M q_m \eta_m^2 L \left(\frac{4}{B}d_m\sigma_s^2 + \frac{(4-B)L^2}{4B} \lambda^2 d_m^3 +\frac{2P}{B} C^2 d_m+2\sigma_{dp}^2 d_m\right) + L\lambda^2 d
\end{align}for the convenience of notation. 

We thus complete the proof.
\end{proof}

\begin{lemma}[Restatement of Lemma \ref{lem:main_2}]
    Under Assumption \ref{assum:Lip_full}-\ref{assum:ind_part_full}, and assume the device delay $\tau_m$ is uniformly upper bounded by $\tau$, we have the following lemma:
    \begin{align*}
        \E\left[V^{t+1}-V^{t}\right]\le-\frac{1-P}{8}\min_m\{q_m\eta_m\} \E\left[\norm{\nabla_{\vw}F(\vw^t)}^2\right]+\gA_1 + \gA_2
    \end{align*}
\end{lemma}
\begin{proof}
Before we give the proof of Lemma \ref{lem:main_2}, we first provide some useful facts that reveal properties of the delayed parameters.

Recall that $\Tilde{\chi}^{t}$ denote the delayed parameters on all devices, and $\chi^{t}$ denote the non-delayed version.
Let $\gF_1 = \E\left[\norm{\chi^{t+1}-\chi^{t}}^2\right]$, $\gF_2 = \E\left[\norm{\Tilde{\chi}^t-\chi^t}^2\right]$. 

For $\gF_1$:
\begin{align*}
    &\E\left[\norm{\chi^{t+1}-\chi^{t}}^2\right]\\
    =&\eta_m^2\E\left[\norm{\breve{G}^t_{m_k}(\Tilde{\vw}^t)}^2\right]\\
    \le&\sum_{m=1}^{M}q_m\eta_m^2\frac{2(1-P)d_{m}}{B}\E\left[\norm{\nabla_{w_{m_k}} F(\vw^t)}^2\right]+\sum_{m=1}^{M}q_m\eta_m^2\frac{(1-P)}{2}\E\left[\norm{\nabla_{w_{m_k}} F_\lambda(\vw^t)}^2\right]\\
    +&\sum_{m=1}^{M}q_m\eta_m^2\left(L^2\E\left[\norm{\Tilde{\chi}^t-\chi^t}^2\right]+\frac{1}{2}\gG_{m}\right)\\
    \le&\sum_{m=1}^{M}q_m\eta_m^2\frac{2(1-P)d_{m}}{B}\E\left[\norm{\nabla_{w_{m_k}} F(\vw^t)}^2\right]+\sum_{m=1}^{M}q_m\eta_m^2\frac{(1-P)}{2}\left(2\E\left[\norm{\nabla_{w_{m_k}} F(\vw^t)}^2\right]+\frac{L^2}{2}\lambda^2 d_m^3\right)\\
    +&\sum_{m=1}^{M}q_m\eta_m^2\left(L^2\E\left[\norm{\Tilde{\chi}^t-\chi^t}^2\right]+\frac{1}{2}\gG_{m}\right)\\
     =&\sum_{m=1}^{M}q_m\eta_m^2\frac{(1-P)(2d_{m}+B)}{B}\E\left[\norm{\nabla_{w_{m_k}} F(\vw^t)}^2\right]+\sum_{m=1}^{M}q_m\eta_m^2\left(\frac{L^2}{4}\lambda^2 d_m^3+L^2\E\left[\norm{\Tilde{\chi}^t-\chi^t}^2\right]+\frac{1}{2}\gG_{m}\right)\numbereq\label{eq:gf1}
\end{align*}
For $\gF_2$, under uniformly bounded delay, we have
\begin{align*}
    &\E\left[\norm{\Tilde{\chi}^t-\chi^t}^2\right]\\
    \le &\E\left[\norm{\sum_{i=1}^{\tau}(\chi^{i+1}-\chi^i)}^2\right]\\
    \le &\tau\sum_{i=1}^{\tau}\E\left[\norm{\chi^{i+1}-\chi^i}^2\right]\numbereq\label{eq:gf2}
\end{align*}
where the last inequality follows by Cauchy-Schwarz Inequality.\\
By the definition of $V^t$:
\begin{align*}
    &\E\left[V^{t+1}-V^{t}\right]\\
    =&\E\left[ F(\vw^{t+1})+\sum_{i=1}^{\tau}\gamma_i\norm{\chi^{t+2-i}-\chi^{t+1-i}}^2\right]-\E\left[ F(\vw^{t})+\sum_{i=1}^{\tau}\gamma_i\norm{\chi^{t+1-i}-\chi^{t-i}}^2\right]\\
    =&\E\left[F(\vw^{t+1})-F(\vw^{t})\right]+\sum_{i=1}^{\tau}\gamma_i\E\left[\norm{\chi^{t+2-i}-\chi^{t+1-i}}^2\right]-\sum_{i=1}^{\tau}\gamma_i\E\left[\norm{\chi^{t+1-i}-\chi^{t-i}}^2\right]\\
    \overset{1)}{\le} & -\sum_{m=0}^M q_m\eta_m(1-P)\left(\frac{1}{4}-\frac{\eta_m L (B+8d_m)}{2B}\right)\E\left[\norm{\nabla_{w_{m}}F(\vw^t)}^2\right]\\
    + &\sum_{m=0}^{M}q_m\eta_m L \left(\frac{1}{2}+2\eta_m L^2\right)\E\left[\norm{\Tilde{\chi}^t-\chi^t}^2\right]+\gA_1\\
    +& \gamma_1\underbrace{\E\left[\norm{\chi^{t+1}-\chi^{t}}^2\right]}_{\gF_1}+\sum_{i=1}^{\tau-1}(\gamma_{i+1}-\gamma_{i})\E\left[\norm{\chi^{t+1-i}-\chi^{t-i}}^2\right]- \gamma_\tau\E\left[\norm{\chi^{t+1-\tau}-\chi^{t-\tau}}^2\right]\\
    \overset{2)}{\le} & -\sum_{m=0}^M q_m\eta_m(1-P)\left(\frac{1}{4}-\frac{\eta_m L (B+8d_m)}{2B}\right)\E\left[\norm{\nabla_{w_{m}}F(\vw^t)}^2\right]\\
    + &\sum_{m=0}^{M}q_m\eta_m L \left(\frac{1}{2}+2\eta_m L^2\right)\E\left[\norm{\Tilde{\chi}^t-\chi^t}^2\right]+\gA_1\\
    +&\sum_{m=1}^{M}q_m\eta_m^2\gamma_1\frac{(1-P)(2d_{m}+B)}{B}\E\left[\norm{\nabla_{w_{m_k}} F(\vw^t)}^2\right]+\sum_{m=1}^{M}q_m\eta_m^2\gamma_1\left(\frac{L^2}{4}\lambda^2 d_m^3+L^2\E\left[\norm{\Tilde{\chi}^t-\chi^t}^2\right]+\frac{1}{2}\gG_{m}\right)\\
    +& \sum_{i=1}^{\tau-1}(\gamma_{i+1}-\gamma_{i})\E\left[\norm{\chi^{t+1-i}-\chi^{t-i}}^2\right]- \gamma_\tau\E\left[\norm{\chi^{t+1-\tau}-\chi^{t-\tau}}^2\right]\\
    \overset{3)}{=} & -\eta_0(1-P)\left(\frac{1}{4}-\frac{\eta_0 L (B+8d_0)}{2B}\right)\E\left[\norm{\nabla_{x_{0}}F(\vw^t)}^2\right]\\
    -&\sum_{m=1}^M q_m\eta_m(1-P)\left(\frac{1}{4}-\frac{\eta_m L (B+8d_m)}{2B}-\frac{\eta_m\gamma_1(2d_{m}+B)}{B}\right)\E\left[\norm{\nabla_{w_{m}}F(\vw^t)}^2\right]\\
    + &\left\{\eta_0 L \left(\frac{1}{2}+2\eta_0 L^2\right)+\sum_{m=1}^{M}q_m\eta_m L \left(\frac{1}{2}+2\eta_m L^2+\eta_m \gamma_1 L\right)\right\}\underbrace{\E\left[\norm{\Tilde{\chi}^t-\chi^t}^2\right]}_{\gF_2}\\
    +&\gA_1 + \sum_{m=1}^{M}q_m\eta_m^2\gamma_1\left(\frac{L^2}{4}\lambda^2 d_m^3+\frac{1}{2}\gG_{m}\right)\\
    +& \sum_{i=1}^{\tau-1}(\gamma_{i+1}-\gamma_{i})\E\left[\norm{\chi^{t+1-i}-\chi^{t-i}}^2\right]- \gamma_\tau\E\left[\norm{\chi^{t+1-\tau}-\chi^{t-\tau}}^2\right]\\
    \overset{4)}{\le} & -\eta_0(1-P)\left(\frac{1}{4}-\frac{\eta_0 L (B+8d_0)}{2B}\right)\E\left[\norm{\nabla_{x_{0}}F(\vw^t)}^2\right]\\
    -&\sum_{m=1}^M q_m\eta_m(1-P)\left(\frac{1}{4}-\frac{\eta_m L (B+8d_m)}{2B}-\frac{\eta_m\gamma_1(2d_{m}+B)}{B}\right)\E\left[\norm{\nabla_{w_{m}}F(\vw^t)}^2\right]\\
    + & \sum_{i=1}^{\tau-1}\left(\gamma_{i+1}-\gamma_{i}+\tau\left(\eta_0 L \left(\frac{1}{2}+2\eta_0 L^2\right)+\sum_{m=1}^{M}q_m\eta_m L \left(\frac{1}{2}+2\eta_m L^2+\eta_m \gamma_1 L\right)\right)\right)\E\left[\norm{\chi^{t+1-i}-\chi^{t-i}}^2\right]\\
    -& \left(\gamma_\tau - \tau\left(\eta_0 L \left(\frac{1}{2}+2\eta_0 L^2\right)+\sum_{m=1}^{M}q_m\eta_m L \left(\frac{1}{2}+2\eta_m L^2+\eta_m \gamma_1 L\right)\right)\right)\E\left[\norm{\chi^{t+1-\tau}-\chi^{t-\tau}}^2\right]\\
    +&\gA_1 + \gA_2\numbereq.\label{eq:vcomplex}
\end{align*}
Above, we used Lemma \ref{lem:main_1} in step $(1)$,  substituted in \eqref{eq:gf1} for $\gF_1$ in step $(2)$, substituted in \eqref{eq:gf2} for $\gF_2$ in step (3), and  defined 
\begin{align*}
    \gA_2 :=& \sum_{m=1}^{M}q_m\eta_m^2\gamma_1\left(\frac{L^2}{4}\lambda^2 d_m^3+\frac{1}{2}\gG_{m}\right)\\
    =&\sum_{m=1}^{M}q_m\eta_m^2\gamma_1\left(\frac{L^2}{4}\lambda^2 d_m^3+\frac{2}{B}d_m\sigma_s^2 + \frac{L^2}{2B} \lambda^2 d_m^3 +\frac{P}{B} C^2 d_m+\sigma_{dp}^2 d_m\right).\numbereq\label{eq:A2}
\end{align*}

From \eqref{eq:vcomplex}, we choose the following relationship for $\gamma_1, \gamma_2, \dots, \gamma_m$:
\begin{align}\label{eq:gamma1}
\gamma_{1} = \frac{\tau^2\left(\eta_0 L \left(\frac{1}{2}+2\eta_0 L^2\right)+\sum_{m=1}^{M}q_m\eta_m L \left(\frac{1}{2}+2\eta_m L^2\right)\right)}{1-\tau^2\sum_{m=1}^{M}q_m\eta_m^2 L^2 }
\end{align}
\[\gamma_{i+1}=\gamma_{i}-\tau\left(\eta_0 L \left(\frac{1}{2}+2\eta_0 L^2\right)+\sum_{m=1}^{M}q_m\eta_m L \left(\frac{1}{2}+2\eta_m L^2+\eta_m \gamma_1 L\right)\right)\]
and we can verify that 
\[\gamma_{\tau}-\tau\left(\eta_0 L \left(\frac{1}{2}+2\eta_0 L^2\right)+\sum_{m=1}^{M}q_m\eta_m L \left(\frac{1}{2}+2\eta_m L^2+\eta_m \gamma_1 L\right)\right)\ge 0\]
We further let
\[\eta_0\le \frac{B}{4L(B+8d_0)}, \eta_m\le \frac{B}{4L(B+8d_0)+8\gamma_1(2d_m+B)},\]
and we finally have 
\begin{align*}
     \E\left[V^{t+1}-V^{t}\right]
    &\le-\frac{\eta_0}{8}(1-P)\E\left[\norm{\nabla_{x_{0}}F(\vw^t)}^2\right]
    -\sum_{m=1}^M \frac{q_m\eta_m}{8}(1-P)\E\left[\norm{\nabla_{w_{m}}F(\vw^t)}^2\right]+\gA_1 + \gA_2\\
    &\le-\frac{1-P}{8}\min_m q_m\eta_m \E\left[\norm{\nabla_{\vw}F(\vw^t)}^2\right]+\gA_1 + \gA_2\numbereq\label{eq:main_2}
\end{align*}
which completes the proof.
\end{proof}
\subsection{Additional Details on Experiment}\label{appen:dataset}
In this section, we would like to give a brief introduction of the datasets and model structures and further justify the choice of the experimental designs below. We follow the experimental settings of existing works\cite{chen2020vafl,xie2024improving,castiglia2023flexible} for a fair comparison.

\textbf{Dataset} The choices of datasets cover a large range of tasks, including:
\begin{itemize}
    \item MNIST and CIFAR-10: standard image benchmarks in machine learning.
    \item ModelNet40: multi-view 3D object classification dataset; each object has 12 views from different angles, which naturally lends itself to feature partitioning across devices in VFL.
    \item Amazon Reviews: sentiment analysis task, which we include to evaluate our algorithm in the NLP domain.
\end{itemize}
\textbf{Data partition and model selection}
For MNIST, we use a two-layer CNN model, for VFL's data partitioning, we split the images by row evenly into 7 sub-images and assign them to 7 devices. For CIFAR-10 we use a four-layer CNN model, and partition each image into 2×2, 4 patches of the same size for 4 devices. For ModelNet40, we use a ResNet-18 model, and partition each object into 12 different camera views and allocate them to 12 devices. For Amazon Reviews, we use a pre-trained BERT~\cite{devlin2018bert} model, and split the tokenized data input into 3 paragraphs of the same number of tokens and distributed them across 3 devices.
For all four datasets, we use a fully connected model of two linear layers with ReLU activations as the server model. 

\end{document}

%% file: symbols_commands.tex
\usepackage{amsmath,amsfonts,bm}

\newcommand\numbereq{\addtocounter{equation}{1}\tag{\theequation}}

\def\va{{\bm{a}}}

\def\vu{{\bm{u}}}

\def\vw{{\bm{w}}}

\def\mY{{\bm{Y}}}

\DeclareMathAlphabet{\mathsfit}{\encodingdefault}{\sfdefault}{m}{sl}
\SetMathAlphabet{\mathsfit}{bold}{\encodingdefault}{\sfdefault}{bx}{n}

\def\gA{{\mathcal{A}}}
\def\gB{{\mathcal{B}}}

\def\gD{{\mathcal{D}}}
\def\gE{{\mathcal{E}}}
\def\gF{{\mathcal{F}}}
\def\gG{{\mathcal{G}}}

\def\gI{{\mathcal{I}}}

\def\gL{{\mathcal{L}}}
\def\gM{{\mathcal{M}}}
\def\gN{{\mathcal{N}}}
\def\gO{{\mathcal{O}}}

\def\gX{{\mathcal{X}}}

\def\sS{{\mathbb{S}}}

\newcommand{\E}{\mathbb{E}}
\newcommand{\Prob}{\mathbb{P}}

\newcommand{\Var}{\mathrm{Var}}

\newcommand{\norm}[1]{\mbox{$\left\lVert #1 \right\rVert$}}